\documentclass[letterpaper]{article} 
\usepackage[preprint]{aaai2027}  
\usepackage{times}  
\usepackage{helvet}  
\usepackage{courier}  
\usepackage[hyphens]{url}  
\usepackage{graphicx} 
\usepackage{natbib}  
\usepackage{caption} 

\usepackage{algorithm}
\usepackage{algorithmic}
\usepackage{amsmath}
\usepackage{amssymb}
\usepackage{amsthm}
\newtheorem{proposition}{Proposition}
\newtheorem{theorem}{Theorem}
\newtheorem{corollary}{Corollary}
\newtheorem{assumption}{Assumption}

\newtheorem{lemma}{Lemma}

\usepackage{subcaption}
\usepackage{booktabs} 
\usepackage{makecell}
\usepackage{newfloat}
\usepackage{listings}
\DeclareCaptionStyle{ruled}{labelfont=normalfont,labelsep=colon,strut=off} 
\floatstyle{ruled}
\newfloat{listing}{tb}{lst}{}
\floatname{listing}{Listing}
\title{A Theoretical Analysis of Memory and Overfitting Phenomena in Stochastic Interpolation Models}

\author{
Yunchen Li, Shaohui Lin, Zhou Yu
}

\affiliations{
East China Normal University\\
\texttt{52284404001@stu.ecnu.edu.cn, shaohuilin007@gmail.com, zyu@stat.ecnu.edu.cn}
}
\usepackage{bibentry}

\begin{document}

\maketitle

\begin{abstract}

This paper provides a theoretical account of memorization in stochastic interpolation models. By leveraging closed-form expressions for the optimal velocity field and the associated score function, we show that, in the continuous-time oracle setting, both deterministic and stochastic generation processes recover training samples. Under Euler discretization, generated samples remain centered around training samples, with deviations controlled by the step size. We further analyze generation in the presence of estimation errors and show that accumulated estimation errors control the endpoint deviation from the training set. These results imply that the generated sample admits a representation as a training sample perturbed by three controlled terms: a discretization-induced bound, an estimation-error-induced bound, and stochastic Gaussian noise. Based on this characterization, we provide theoretical definitions of overfitting and underfitting in generative models. Synthetic simulations support our theoretical findings.

\end{abstract}

\section{Introduction}
Generative models have achieved remarkable success across a wide range of domains, including image generation~\citep{rombach2022high,peebles2023scalable}, image restoration~\citep{wu2024seesr,wu2024one}, and text generation~\citep{nie2025large,gong2024scaling}. Despite these empirical advances, prior studies have shown that generative models may exhibit data-copying behavior by producing samples that closely resemble training instances. This phenomenon, commonly referred to as memorization, becomes particularly pronounced in limited-data regimes. Fig. 1 illustrates this phenomenon using ImageNet, where models trained on small subsets generate samples that closely match their nearest training examples.

\begin{figure*}[!t]
\centering
\begin{subfigure}[b]{0.48\textwidth}
\includegraphics[width=\textwidth]{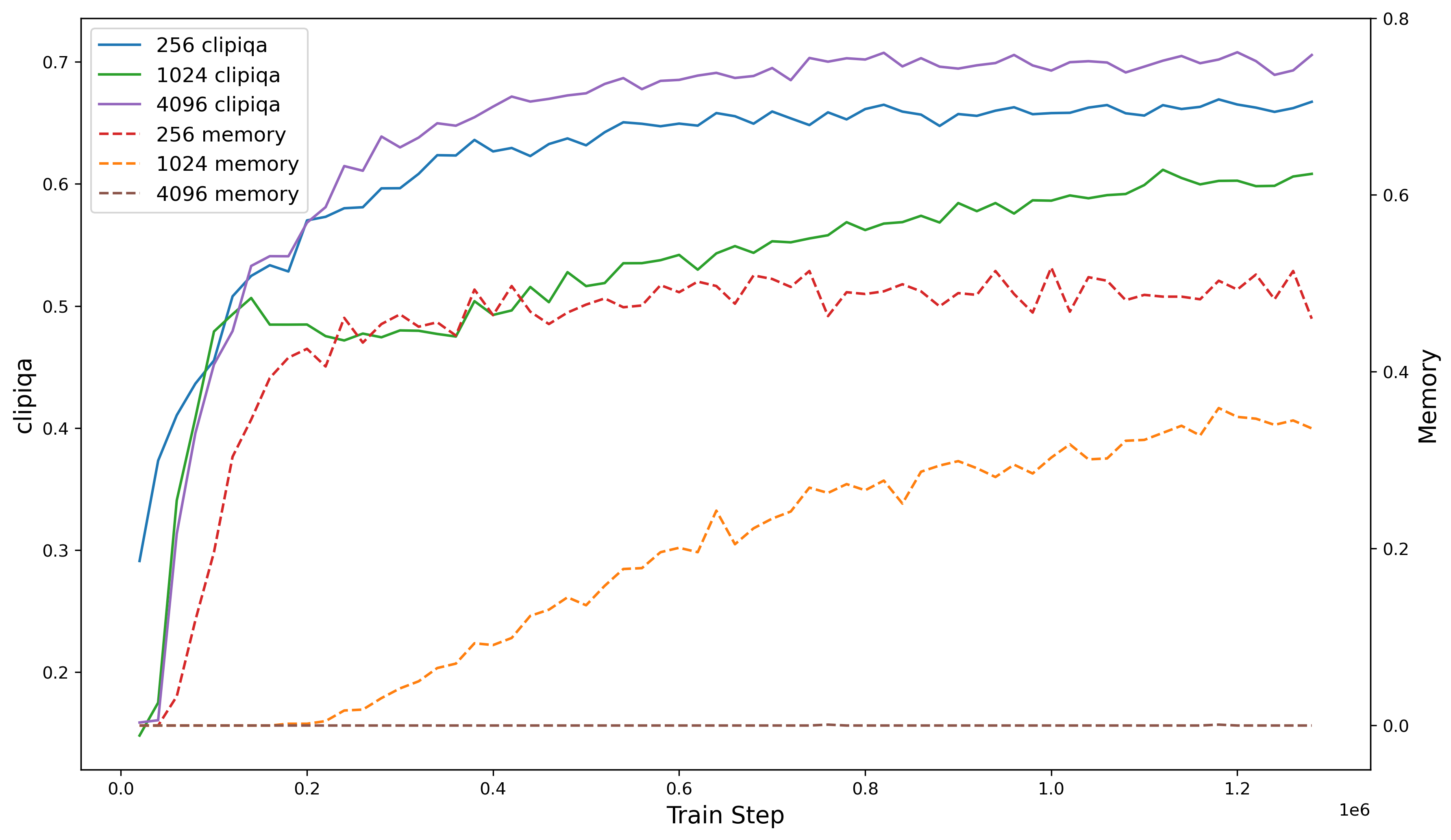}
\caption{}
\label{fig:generate_and_memory1}
\end{subfigure}
\hfill
\begin{subfigure}[b]{0.48\textwidth}
\includegraphics[width=\textwidth]{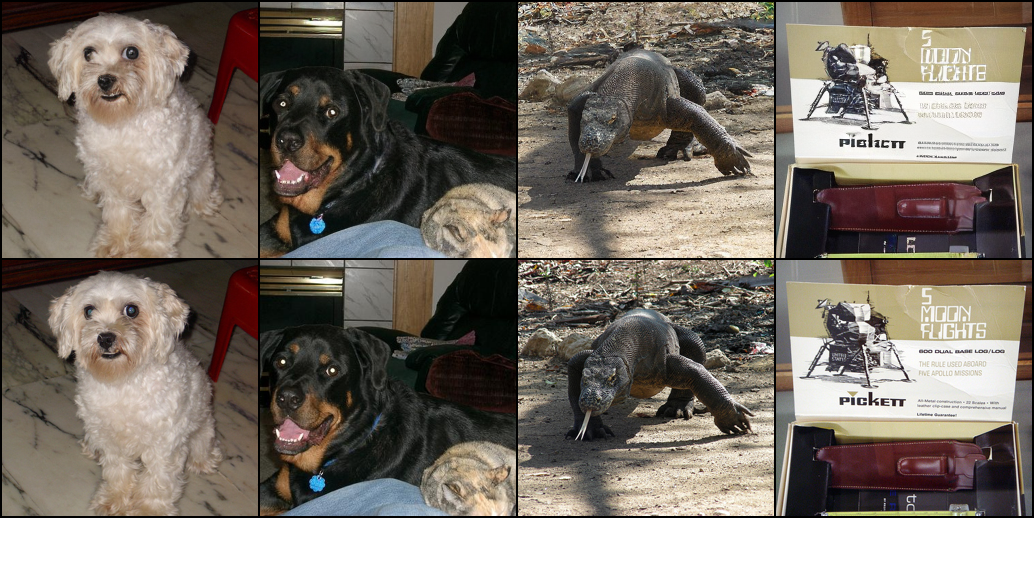}
\caption{}
\label{fig:generate_and_memory2}
\end{subfigure}
\caption{Empirical memorization phenomenon on ImageNet. We train SiT-B/2 models on randomly sampled ImageNet subsets with 256, 1024, and 4096 images. Generation quality is measured by CLIPIQA, and memorization is evaluated following~\citet{yoon2023diffusion}. (a) Generation quality and memorization rate during training. (b) Generated samples and their nearest training samples, illustrating copy-like behavior under limited-data training.}
\label{fig:generate_and_memory}
\end{figure*}

A growing body of empirical work has documented memorization in generative models~\citep{yoon2023diffusion,gu2023memorization,somepalli2023understanding,bonnaire2025diffusion,kim2025how,scardace2026novel,vu2026memorization,shi2026closer,rao2026generalization}. These studies suggest that memorization is not an isolated artifact, but rather a recurring phenomenon across diffusion-based models. Nevertheless, most existing works are primarily empirical, leaving the generative dynamics responsible for memorization insufficiently characterized.

Recent theoretical studies have examined memorization and generalization from complementary perspectives. Some works analyze optimal learning objectives and theoretical separations between memorization and generalization~\citep{li2024good,bertrand2025closed,buchanan2026edge,ye2025provable}, while others focus on data geometry and neural network capacity~\citep{li2023generalization,kadkhodaie2023generalization,gao2024flow,halder2024memorization,shen2026manifold,liu2026navigation}. Related analyses further connect memorization to implicit overfitting mechanisms and structural properties of empirical score functions~\citep{song2025selective,dodson2026two,kaiser2026diffusion,zhou2026smoothing}. Despite these advances, a unified dynamical account of how memorization arises during generation remains incomplete. In particular, existing analyses have not yet provided a systematic treatment of both deterministic and stochastic generation mechanisms, nor a precise characterization of how velocity and score estimation errors propagate to the generated endpoint.

In this paper, we study memorization through the lens of stochastic interpolation models~\citep{albergo2023stochastic}, which provide a unified formulation for diffusion-based generative models~\citep{ho2020denoising,song2020score,lipman2022flow}. In the continuous-time oracle regime, we show that both deterministic and stochastic generation processes exactly recover training samples. We further analyze Euler-discretized sampling with step size $h$. In this setting, deterministic generation recovers a training sample up to an $h$-controlled discretization error, while stochastic generation additionally introduces a Gaussian perturbation, consistent with the Gaussian-blurring phenomenon observed in~\citet{li2024good}. These oracle results reveal a strong finite-sample memorization effect inherent to stochastic interpolation models.

We then analyze generation in the presence of estimation errors. In practice, the learned velocity field and score function generally deviate from their oracle counterparts, and these deviations accumulate along the reverse sampling trajectory. We show that the endpoint deviation from the training set is controlled by the corresponding accumulated training error. Consequently, each generated sample can be characterized as a selected training sample affected by a discretization-induced bound, an estimation-error-induced bound, and stochastic Gaussian noise. This decomposition further enables us to formulate theoretical definitions of overfitting and underfitting in generative models: overfitting corresponds to generated samples concentrating excessively near the training set, whereas underfitting corresponds to samples deviating substantially from the data distribution.

The contributions of this paper are summarized as follows:
\begin{itemize}
\item We analyze oracle generation in stochastic interpolation models. In the error-free continuous-time setting, both deterministic and stochastic generation processes recover training samples; under Euler discretization, the distance to the training set is controlled by the step size $h$.

\item We quantify how velocity and score estimation errors propagate through the generation process and use this analysis to formulate theoretical definitions of overfitting and underfitting in generative models.

\item We characterize generated samples in stochastic interpolation models by decomposing each sample into a selected training sample, a discretization-induced error, an estimation-error-induced error, and, in the stochastic case, an additional Gaussian noise term.
\end{itemize}

\section{Preliminary: Stochastic Interpolation}

In this section, we introduce the stochastic interpolation model \citep{albergo2023stochastic}. Given \(Z_0\sim\rho_0\) and \(Z_1\sim\rho_1\), stochastic interpolation defines a time-dependent random variable
\begin{equation}
	Z_t = \mathcal{I}(t,Z_0,Z_1) + \gamma(t) {\eta}, \; t \in [0,1],
	\label{stochastic interpolation}
\end{equation}
where \({\eta}\sim\mathcal N(0,I_d)\) is independent of $Z_0$ and $Z_1$, \(\mathcal I(0,Z_0,Z_1)=Z_0\), \(\mathcal I(1,Z_0,Z_1)=Z_1\), and \(\gamma(0)=\gamma(1)=0\) with \(\gamma(t)\ge0\) for \(t\in(0,1)\).

Based on the Fokker–Planck equation, the velocity field $b({z},t)$ and the score function $s({z},t)$ are defined as
\begin{align}
	b({z},t) &=\mathbb{E}\Big[\frac{\partial}{\partial t} \mathcal{I}(t,Z_0,Z_1)+{\gamma}'(t){\eta}|Z_t={z}\Big],\label{population_b} \\
	s( z,t) &= \nabla_{ z} \log\rho_t( z).\label{population_s}
\end{align} 

This induces the deterministic generation process
\begin{equation}
	dZ_t = {b}(Z_t,t)dt, \label{determinstic_generation}
\end{equation}
and the stochastic generation process
\begin{equation}
	dZ_t = \Big({b}(Z_t,t)-\zeta(t){s}(Z_t,t)\Big)dt + \sqrt{2\zeta(t)}dW_t. \label{stochastic_generation}
\end{equation}

Starting from $Z_1\sim\rho_1$, both processes generate samples satisfying $Z_0\sim\rho_0$.

In practical applications, the explicit forms of $b( z,t)$ and $s( z,t)$ are typically unknown and should be estimated by minimizing the following population objectives
\begin{equation*}
	\begin{split}
		&\hat{b}({z},t)=\arg\min_b\mathbb{E}\Big\Vert b(Z_t,t) - \Big(\frac{\partial}{\partial t} \mathcal{I}(t,Z_0,Z_1)+{\gamma}'(t){\eta}\Big) \Big\Vert^2,\\
		&\hat{s}({z},t)=\arg\min_s\mathbb{E}\Big\Vert s(Z_t,t)-\frac{1}{\gamma(t)}{\eta}\Big\Vert^2.
	\end{split}
\end{equation*}

A widely used choice of stochastic interpolation is $\mathcal{I}(t,Z_0,Z_1)=\alpha(t)Z_0+\beta(t)Z_1$, which satisfies $\alpha(0)=1,\,\alpha(1)=0,\,\beta(0)=0,\,\beta(1)=1$. Without loss of generality, we assume $\alpha(t)>0$ and $\beta(t)>0$ when $t\in(0,1)$. Under this choice, $Z_t$ can be expressed as
\begin{equation*}
	Z_t = \alpha(t)Z_0+\beta(t)Z_1 + \gamma(t)\eta.
\end{equation*}

The stochastic interpolation model includes flow matching when \(\gamma\equiv0\) and suitable score-based models under Gaussian endpoints and appropriate schedules.


\subsection{Setting and Notation}

The target distribution is defined as the empirical law
\begin{equation}
	\rho_0=\frac{1}{n}\sum_{i=1}^n\delta_{X_i}.
	\label{eq:empirical_rho0}
\end{equation}
Let \(\mathcal X:=\{X_1,\dots,X_n\}\), \(\operatorname{dist}(z,\mathcal X):=\min_i\|z-X_i\|\), \(D_X:=\max_{i,j}\|X_i-X_j\|\), and \(M_X:=\max_i\|X_i\|\). Here \(D_X\) is the diameter of the finite training set and \(M_X\) is its maximum norm. The default source distribution is \(\rho_1=\mathcal N(0,I_d)\). We use the linear interpolation
\[
Z_t=\alpha(t)Z_0+\beta(t)Z_1+\gamma(t)\eta,
\]
where \(Z_0\sim\rho_0\), \(Z_1\sim\rho_1\), and
\(\eta\sim\mathcal N(0,I_d)\) are independent. Define
\[
\begin{aligned}
	C_1(t)&=\gamma(t)\gamma'(t)+\beta'(t)\beta(t),\\
	C_2(t)&=[\gamma(t)\gamma'(t)+\beta'(t)\beta(t)]\alpha(t)-[\gamma^2(t)+\beta^2(t)]\alpha'(t),\\
	C_3(t)&=\gamma^2(t)+\beta^2(t),
\end{aligned}
\]
with all schedule functions evaluated at the same time \(t\).
\section{Optimal Generation and Discrete Sampling}
\label{sec:theory}

In this section, we first analyze the error-free oracle regime and then study
how numerical discretization affects the generation dynamics. The goal is to
identify the finite-sample memorization mechanism induced by the oracle
velocity field and to characterize how deterministic and stochastic samplers
behave under Euler discretization.

\subsection{Continuous-Time Oracle Generation}
\label{sec:continuous_oracle_generation}

We first derive closed-form expressions for the oracle velocity field and the oracle score function. These
quantities reveal why finite training sets induce an attracting structure around
the observed samples.

\begin{proposition}
	The oracle velocity field and oracle score have the following closed forms:
	\begin{equation}
		\begin{split}
			b^*(z,t)=\sum_{i=1}^n\frac{C_1(t)z-C_2(t)X_i}{C_3(t)}\omega_i(z,t),
		\end{split}
		\label{eq:oracle_b}
	\end{equation}
	where $\omega_i(z,t):=\frac{\exp\{-\|z-\alpha(t)X_i\|^2/(2C_3(t))\}}{\sum_{j=1}^n\exp\{-\|z-\alpha(t)X_j\|^2/(2C_3(t))\}}.$
	
	The oracle score satisfies
	\begin{equation}
		s^*(z,t)=\frac{\alpha(t)}{B(t)}b^*(z,t)-\frac{\alpha'(t)}{B(t)}z,
		\label{eq:oracle_s}
	\end{equation}
	where $B(t)=\beta(t)\bigl[\alpha'(t)\beta(t)-\alpha(t)\beta'(t)\bigr]+\gamma(t)\bigl[\gamma(t)\alpha'(t)-\gamma'(t)\alpha(t)\bigr].$
    
\end{proposition}

Proposition 1 shows that the oracle velocity field is determined by softmax weights over the training samples, and the oracle score function is a linear combination of the velocity field and the current state \(z\). As \(t \to 0\),
\(C_3(t)\) vanishes, so the softmax weights concentrate on the nearest training
sample. Therefore, the oracle dynamics is driven toward empirical samples. This leads
to memorization in both deterministic and stochastic generation, motivating the following theorem.
\begin{theorem}
	The oracle generation dynamics recover the empirical target distribution.
	For deterministic generation defined in Eq.~\ref{determinstic_generation}, there exists \(i\in\{1,\dots,n\}\) such that
	\(Z_0=X_i\). 
	
	For stochastic generation defined in Eq.~\ref{stochastic_generation}, for \(\varepsilon\in(0,1)\), there exist
	\(i\in\{1,\dots,n\}\) and
	\(\xi_\varepsilon\sim\mathcal N(0,C_3(\varepsilon)I_d)\), independent of \(i\),
	such that
	\[
	Z_\varepsilon
	\stackrel{d}{=}
	\alpha(\varepsilon)X_i+\xi_\varepsilon.
	\]
\end{theorem}

This theorem shows that deterministic oracle generation exactly selects a
training sample, while stochastic oracle generation at a finite $\epsilon>0$ produces Gaussian clouds centered at scaled training samples. As \(\varepsilon\to0\),
\(C_3(\varepsilon)\to0\) and \(\alpha(\varepsilon)\to1\), so these Gaussian
clouds collapse to the empirical distribution. 

\subsection{Euler-Discretized Oracle Generation}
\label{sec:euler_oracle_generation}

We now study the practical setting where generation is implemented with a
finite-step numerical discretization. Let
\[
t_k=kh,\qquad k=0,1,\dots,K,\qquad h=\frac1K.
\]
We consider the backward Euler update
\begin{equation}
	Z_{k-1}^h
	=
	Z_k^h-hF(Z_k^h,t_k)+\sigma_k\xi_k,
	\qquad
	k=K,K-1,\dots,1,
	\label{eq:euler_template}
\end{equation}
where the choice of \(F\) depends on the sampling scheme and \(\xi_k\sim\mathcal N(0,I_d)\) are independent whenever \(\sigma_k>0\).

For any \(z\in\mathbb R^d\), define
\[
i_k(z)\in
\arg\min_{1\le i\le n}\|z-\alpha(t_k)X_i\|,
\]
which denotes the training sample whose scaled center
\(\alpha(t_k)X_i\) is closest to \(z\) at time \(t_k\). Define the squared
margin
\[
\begin{aligned}
	m_k(z):=
	\min_{j\neq i_k(z)}
	\{
	\|z-\alpha(t_k)X_j\|^2
	-\|z-\alpha(t_k)X_{i_k(z)}\|^2
	\},
\end{aligned}
\]
which measures the separation between the nearest and second-nearest scaled training centers.

For the deterministic Euler scheme, we take \(F=b^*\) and \(\sigma_k=0\) in
Eq.~\ref{eq:euler_template}. Along the trajectory, define
\[
d_{i,k}^h:=\|Z_k^h-\alpha(t_k)X_i\|^2,
\]
which is the squared distance between the current Euler iterate \(Z_k^h\) and
the scaled training center \(\alpha(t_k)X_i\). Define
\[
w_{i,k}^h
:=
\frac{
	\exp\{-d_{i,k}^h/(2C_3(t_k))\}
}{
	\sum_{\ell=1}^n
	\exp\{-d_{\ell,k}^h/(2C_3(t_k))\},
} 
\]
which defines the softmax weights induced by the oracle velocity field. Finally, define $\bar X_k^h:=\sum_{i=1}^n w_{i,k}^hX_i$ as the softmax-weighted training sample induced by the oracle field at the current Euler step.
Substituting Eq.~\ref{eq:oracle_b} into the deterministic Euler update gives
\begin{equation}
	Z_{k-1}^h
	=
	\lambda_k Z_k^h+c_k\bar X_k^h,
	\label{eq:euler_deterministic_linear}
\end{equation}
where \(a_k:=hC_1(t_k)/C_3(t_k)\), \(c_k:=hC_2(t_k)/C_3(t_k)\), and \(\lambda_k:=1-a_k\).

The following assumption rules out unbounded terminal entry points and points near softmax decision boundaries, where the selected training sample is ambiguous.

\begin{assumption}
	\label{ass:hp_terminal_event}
	For \(\delta\in(0,1)\), there exists an event \(\mathcal A_h(\delta)\)
	such that $\mathbb P(\mathcal A_h(\delta))\ge 1-\delta.$ On \(\mathcal A_h(\delta)\),
	\[
	\|Z_K^h\|\le R_\delta,
	\qquad
	m_k(Z_k^h)\ge u_k,\quad k=1,\dots,K.
	\]
	Here \(R_\delta\) controls the size of the entry point into the terminal regime,
	and \(u_k\) controls the softmax margin at time \(t_k\).
\end{assumption}
Since \(Z_K^h\sim\mathcal N(0,I_d)\), a valid choice is
\[
R_\delta=\sqrt d+\sqrt{2\log(1/\delta)}.
\]

We introduce the following three error terms. The terminal-entry error measures how much the terminal entry point \(Z_K^h\) can still affect the endpoint after being contracted by the Euler dynamics. 
\[
B_{h,\mathrm{init}}^{\mathrm{det}}(\delta):=\left(\prod_{k=1}^K|\lambda_k|\right)(R_\delta+M_X).
\]
The schedule-mismatch error captures the mismatch between the affine update of the scaled centers and the target scaled centers at the next Euler time. 
\[
B_{h,\mathrm{sch}}^{\mathrm{det}}:=\sum_{j=1}^K\left(\prod_{\ell=1}^{j-1}|\lambda_\ell|\right)
\left|\lambda_j\alpha(t_j)+c_j-\alpha(t_{j-1})\right|M_X.
\]
The softmax-selection error measures the residual error caused by the softmax weights not being exactly one-hot; this term becomes small when the margin \(u_j\) is large relative to the variance scale \(C_3(t_j)\).
\[
B_{h,\mathrm{sm}}^{\mathrm{det}}:=\sum_{j=1}^K\left(\prod_{\ell=1}^{j-1}|\lambda_\ell|\right)|c_j|D_X(n-1)\
\exp\left(-\frac{u_j}{2C_3(t_j)}\right).
\]
The overall error is then defined as the sum of these three terms:
\[
B_h^{\mathrm{det}}(\delta):=B_{h,\mathrm{init}}^{\mathrm{det}}(\delta)+B_{h,\mathrm{sch}}^{\mathrm{det}}+B_{h,\mathrm{sm}}^{\mathrm{det}}.
\]

\begin{theorem}
	\label{thm:hp_det_euler}
	Under Assumption~\ref{ass:hp_terminal_event}, for every \(\delta\in(0,1)\),
	\[
	\mathbb P\left(
	\operatorname{dist}(Z_0^h,\mathcal X)
	\le
	B_h^{\mathrm{det}}(\delta)
	\right)
	\ge 1-\delta.
	\]
\end{theorem}

Theorem~\ref{thm:hp_det_euler} shows that the endpoint distance to the training set is controlled by the three error terms above. In the classical schedule, these terms admit a simpler form, leading to the following corollary.
\begin{corollary}
	\label{cor:simple_det_euler_bound_sqrt_gamma}
	Under Assumption~\ref{ass:hp_terminal_event}, consider the interpolation schedule $\alpha(t)=1-t$, $\beta(t)=t$ and $
	\gamma(t)=\sqrt{t(1-t)}$. Suppose further that there exists a constant \(\theta\in(0,1)\), independent of
	\(h\), such that for all \(j\) satisfying \(t_j\le \theta\), $u_j\ge2t_j\log\left(\frac{1}{\sqrt h}\right)$. Then, for every \(\delta\in(0,1)\),
	
	\[
	\mathbb P\left(
	\operatorname{dist}(Z_0^h,\mathcal X)
	\lesssim
	\sqrt h
	\left(
	1+\sqrt d+\sqrt{\log(1/\delta)}
	\right)
	\right)
	\ge
	1-\delta.
	\]
\end{corollary}

Thus, under Euler-based sampling, the generated sample remains close to the training set and its distance to the training set is controlled by the step size \(h\) with high probability.

We next consider the Euler discretization of the stochastic oracle
generation. This is obtained from Eq.~\ref{eq:euler_template} by taking
\[
F=b^*-\zeta s^*,
\qquad
\sigma_k=\sqrt{2\zeta(t_k)h}.
\]
In the terminal region, the oracle stochastic update can be written as
\begin{equation}
	Z_{k-1}^h
	=
	r_kZ_k^h+q_k\bar X_k^h+\sigma_k\xi_k,
	\label{eq:euler_maruyama_linear}
\end{equation}
where

\[
r_k
:=
1
-
h\left[
\left(
1-\frac{\zeta(t_k)\alpha(t_k)}{B(t_k)}
\right)
\frac{C_1(t_k)}{C_3(t_k)}
+
\frac{\zeta(t_k)\alpha'(t_k)}{B(t_k)}
\right],
\]
\[
q_k
:=
h\left(
1-\frac{\zeta(t_k)\alpha(t_k)}{B(t_k)}
\right)
\frac{C_2(t_k)}{C_3(t_k)}.
\]
Here \(r_k\) is the coefficient that propagates the current iterate
\(Z_k^h\), while \(q_k\) is the coefficient of the softmax-weighted training
sample \(\bar X_k^h\). 

Define \(\Pi_{j-1}:=\prod_{\ell=1}^{j-1}r_\ell\) and \(G_h:=\sum_{j=1}^K\Pi_{j-1}\sigma_j\xi_j\). Then \(G_h\sim\mathcal N(0,\Sigma_h)\), where $\Sigma_h=\sum_{j=1}^K\Pi_{j-1}^2\sigma_j^2I_d$.

We additionally require that the selected softmax center does not switch within the terminal region. This stability condition allows us to use a single training sample as the center of the final decomposition. This assumption is imposed mainly for proof simplicity. Similar bounds can be obtained without it by tracking possible selector switches. Since \(C_3(t)\to 0\) as \(t\to 0\), the softmax weights concentrate on the nearest scaled training center, so such stability naturally holds in the terminal regime away from decision boundaries.

\begin{assumption}
	\label{ass:hp_sto_terminal}
	On the event \(\mathcal A_h(\delta)\) from Assumption~\ref{ass:hp_terminal_event},
	there exists a random index \(i_h\in\{1,\dots,n\}\) such that
	\[
	i_k(Z_k^h)=i_h,
	\qquad
	k=1,\dots,K.
	\]
\end{assumption}

Similar to deterministic generation, we decompose the stochastic generation error into three parts. 
\[
\begin{gathered}
	B_{h,\mathrm{init}}^{\mathrm{sto}}(\delta):=
	\left(\prod_{k=1}^K |r_k|\right)(R_\delta+M_X),\\
	B_{h,\mathrm{aff}}^{\mathrm{sto}}
	:=
	\sum_{j=1}^K
	\left(
	\prod_{\ell=1}^{j-1}|r_\ell|
	\right)
	|r_j+q_j-1|M_X,\\
	B_{h,\mathrm{sm}}^{\mathrm{sto}}
	:=
	\sum_{j=1}^K
	\left(
	\prod_{\ell=1}^{j-1}|r_\ell|
	\right)
	|q_j|D_X(n-1)
	\exp\left(
	-\frac{u_j}{2C_3(t_j)}
	\right),\\
\end{gathered}
\]
The overall non-Gaussian stochastic Euler error is then defined as the sum of these three terms:
\[
B_h^{\mathrm{sto}}(\delta):=
B_{h,\mathrm{init}}^{\mathrm{sto}}(\delta)+B_{h,\mathrm{aff}}^{\mathrm{sto}}
+B_{h,\mathrm{sm}}^{\mathrm{sto}}.
\]
\begin{theorem}
	\label{thm:sto_gaussian_error}
	Under Assumptions~\ref{ass:hp_terminal_event} and~\ref{ass:hp_sto_terminal}, for every \(\delta\in(0,1)\), with probability at least \(1-\delta\), there exists a random index \(i\in\{1,\dots,n\}\) such that
	\[
	Z_0^h=X_{i}+G_h+E_h,
	\]
	where \(G_h\sim\mathcal N(0,\Sigma_h)\) and $\|E_h\|\le B_h^{\mathrm{sto}}(\delta)$.
\end{theorem}

Theorem~\ref{thm:sto_gaussian_error} is the stochastic analogue of Theorem~\ref{thm:hp_det_euler}. It separates the stochastic Euler output into a selected training sample, an explicit propagated Gaussian sampling term, and a controlled non-Gaussian remainder. In the classical schedule, these terms admit a simpler form, leading to the following corollary.

\begin{corollary}
	\label{cor:simple_sto_euler_bound_sqrt_gamma}
	Assume the conditions of Theorem~\ref{thm:sto_gaussian_error}. Consider the
	schedule \(\alpha(t)=1-t\), \(\beta(t)=t\), \(\gamma(t)=\sqrt{t(1-t)}\), and \(\zeta(t)=\sqrt{t(1-t)}\). Suppose further that there exists a constant \(\theta\in(0,1)\), independent of
	\(h\), such that for all \(j\) satisfying \(t_j\le \theta\), $u_j\ge2t_j\log\left(\frac1{\sqrt h}\right)$. Then, for every \(\delta\in(0,1)\), with probability at least \(1-\delta\),
	there exists \(i_h\in\{1,\dots,n\}\) such that
	\[
	Z_0^h=X_{i_h}+\mathcal N(0,\Sigma_h)+E_h,
	\]
	where
	\[
	\|E_h\|\lesssim \sqrt h\left(1+\sqrt d+\sqrt{\log(1/\delta)}
	\right).
	\]
\end{corollary}

The stochastic Euler result reveals a noisy form of finite-data memorization. Specifically, stochastic Euler sampling generates samples around training-sample centers, with deviations consisting of propagated Gaussian noise and a controlled non-Gaussian error.

\section{Generation under Estimation Error}
\label{sec:error_generation}

We now move beyond the oracle regime and study how estimation error affects the generated dynamics. In practice, the learned velocity field and score function do not exactly coincide with their oracle counterparts. This discrepancy plays a central role in determining whether the model memorizes the training samples, generates meaningful perturbations around them, or fails to produce samples consistent with the data distribution.

For the Euler trajectory under consideration, let \(\nu_k\) be the law of its \(k\)-th iterate \(Z_k^h\). 
We make the following concentrability assumption to transfer estimation-error bounds under the interpolation law \(\rho_{t_k}\) to the generated trajectory.
\begin{assumption}
	\label{ass:concentrability}
	There exists \(\Gamma>0\) such that
	\[
	\nu_k\ll\rho_{t_k},
	\qquad
	\frac{d\nu_k}{d\rho_{t_k}}\le\Gamma,
	\qquad
	k=1,\dots,K.
	\]
\end{assumption}

This condition rules out generation trajectories that concentrate in regions
nearly invisible under the interpolation distribution. 


\subsection{Deterministic Generation under Velocity Estimation Error}
\label{sec:det_error_generation}

We now study deterministic generation with an estimated velocity field \(\widehat b(z,t)=b^*(z,t)+\epsilon(z,t)\). Replacing \(F=b^*\) by \(F=\widehat b\) in Eq.~\ref{eq:euler_template} gives
\[
Z_{k-1}^h
=
\lambda_k Z_k^h+c_k\bar X_k^h
-
h\epsilon(Z_k^h,t_k).
\]

To relate this accumulated error to the training objective, define the
deterministic propagation weights $\Pi_{j-1}:=\prod_{\ell=1}^{j-1}\lambda_\ell$ and the discrete population training error
\begin{equation}
	\mathcal L_{\mathrm{train}}^h:=h\sum_{k=1}^K\mathbb E_{X\sim\rho_{t_k}}
	\|\epsilon(X,t_k)\|^2.
	\label{eq:train_error_det}
\end{equation}
Also define \(A_h:=h\sum_{j=1}^K\Pi_{j-1}^2\).

\begin{theorem}
	\label{thm:det_training_error_memorization}
	Under Assumptions~\ref{ass:hp_terminal_event} and~\ref{ass:concentrability}, for every
	\(\delta,\eta\in(0,1/2)\), with probability at least \(1-\delta-\eta\),
	\[
	\operatorname{dist}(Z_0^h,\mathcal X)
	\le
	B_h^{\mathrm{det}}(\delta)
	+
	\sqrt{
			\frac{\Gamma A_h\mathcal L_{\mathrm{train}}^h}{\eta}
	}.
	\]
	In particular, take \(\eta=\delta\). Under the same conditions as in
	Corollary~\ref{cor:simple_det_euler_bound_sqrt_gamma}, with probability
	at least \(1-2\delta\),
	\[
    \begin{split}
        \operatorname{dist}(Z_0^h,\mathcal X)
	&\lesssim
	\sqrt h\bigl(1+\sqrt d+\sqrt{\log(1/\delta)}\bigr)\\
	&+
	\sqrt{
		\frac{\Gamma h\log(e/h)\mathcal L_{\mathrm{train}}^h}{\delta}
	}.
    \end{split}
	\]
\end{theorem}

Theorem~\ref{thm:det_training_error_memorization} provides a direct
training-error-based explanation of overfitting. When the learned velocity
field fits the oracle velocity field well along the interpolation path,
\(\mathcal L_{\mathrm{train}}^h\) is small, forcing the generated sample to
remain close to the finite training set with high probability. Thus, in the
finite-sample regime, small training error may reflect overfitting to the
empirical data distribution rather than genuine generalization.

For the converse direction, we define the propagated deterministic training error as
\[
\mathcal L_{\mathrm{prop}}^h
:=
h\sum_{j=1}^K
\Pi_{j-1}^2
\mathbb E_{\rho_{t_j}}
\|\epsilon(X,t_j)\|^2.
\]
This quantity measures the training error after accounting for the contraction
weights with which errors at different time steps are transferred to the final
output.

A large propagated error does not automatically imply a large endpoint
deviation, since errors from different time steps may cancel and the generated
trajectory may avoid regions where the estimation error is large. We therefore
impose the following assumption.
For this underfitting statement, let
\(\mathbb E_{\delta}[\cdot]=\mathbb E[\cdot\mid\mathcal A_h(\delta)]\), and let
\(\nu_{j,\delta}\) be the law of \(Z_j^h\) under this conditional distribution.

\begin{assumption}
	\label{ass:det_non_cancellation}
	Let $\mathcal E_h:=\sum_{j=1}^K\Pi_{j-1}h\,\epsilon(Z_j^h,t_j).$ There exist constants \(\kappa,\gamma>0\) such that
	\[
	\mathbb E_{\delta}\|\mathcal E_h\|^2
	\ge
	\kappa
	h\sum_{j=1}^K
	\Pi_{j-1}^2
	\mathbb E_{\nu_{j,\delta}}\|\epsilon(Z,t_j)\|^2,
	\]
	and
	\[
	\mathbb E_{\nu_{j,\delta}}\|\epsilon(Z,t_j)\|^2
	\ge
	\gamma
	\mathbb E_{\rho_{t_j}}\|\epsilon(X,t_j)\|^2,
	\qquad j=1,\dots,K.
	\]
\end{assumption}

The first inequality prevents strong cancellation of propagated errors across
time steps. The second ensures that the generated trajectory encounters regions
where the estimation error is comparable to that under the interpolation
distribution. 

By unrolling the deterministic Euler recursion with
\(\widehat b=b^*+\epsilon\), the final sample can be written as
\[
Z_0^h
=
X_{i}
-
\mathcal E_h
+
R_h,
\]
where \(R_h\) is the oracle deterministic Euler residual satisfying $\|R_h\|\le B_h^{\mathrm{det}}(\delta)$. This decomposition allows us to convert a lower bound on the propagated training error into a lower bound on the endpoint deviation from the selected training sample.

\begin{theorem}
	\label{thm:det_training_error_underfitting}
	Under Assumptions~\ref{ass:hp_terminal_event} - \ref{ass:det_non_cancellation}, if
	\[
	\mathcal L_{\mathrm{prop}}^h
	\ge
	\frac{
		\bigl(\tau+B_h^{\mathrm{det}}(\delta)\bigr)^2
	}{
		\kappa\gamma
	},
	\]
	then
	\[
	\left(
	\mathbb E_{\delta}\|Z_0^h-X_{i_h}\|^2
	\right)^{1/2}
	\ge
	\tau.
	\]	
\end{theorem}

Theorem~\ref{thm:det_training_error_underfitting} gives a complementary
interpretation of underfitting. When the training error is large and its
trajectory-wise contributions do not cancel, the accumulated error shifts the
final output away from the oracle-selected training sample. This means that the
learned velocity field fails to reproduce the finite-sample attractor.

\subsection{Stochastic Generation under Velocity and Score Estimation Errors}
\label{sec:sto_error_generation}

We now consider stochastic generation when both the velocity field and the score
function are estimated. Let \(\widehat b=b^*+\epsilon_b\) and
\(\widehat s=s^*+\epsilon_s\). Since stochastic generation uses the drift
\(b-\zeta s\), the relevant effective error is
\begin{equation}
	\epsilon_F(z,t)
	:=
	\epsilon_b(z,t)-\zeta(t)\epsilon_s(z,t).
	\label{eq:effective_stochastic_error}
\end{equation}

To relate the accumulated effective error to the training objective, define the
effective stochastic training error
\begin{equation}
	\mathcal L_{\mathrm{sto}}^h:=h\sum_{k=1}^K
	\mathbb E_{\rho_{t_k}}
	\|\epsilon_F(X,t_k)\|^2.
	\label{eq:sto_train_error}
\end{equation}
The accumulated effective error along the stochastic trajectory is
\[
\mathcal E_{F,h}
:=
\sum_{j=1}^K
\Pi_{j-1}h\,\epsilon_F(Z_j^h,t_j).
\]

\begin{theorem}
	\label{thm:sto_training_error_memorization}
	Under Assumptions~\ref{ass:hp_sto_terminal} and~\ref{ass:concentrability}, for every
	\(\delta,\eta,\eta_G\in(0,1/3)\), with probability at least
	\(1-\delta-\eta-\eta_G\),
	\[
	\begin{split}
		\operatorname{dist}(Z_0^h,\mathcal X)
		\le\;&
		B_h^{\mathrm{sto}}(\delta)
		+
		\sqrt{
			\frac{\Gamma A_h\mathcal L_{\mathrm{sto}}^h}{\eta}
		}\\
		&+
		\sqrt{\operatorname{Tr}(\Sigma_h)}
		+
		\sqrt{2\|\Sigma_h\|_{\mathrm{op}}\log(1/\eta_G)}.
	\end{split}
	\]
	
	Under the same conditions as in Corollary~\ref{cor:simple_sto_euler_bound_sqrt_gamma}, with probability at least \(1-3\delta\),
	\[
	\begin{split}
		\operatorname{dist}(Z_0^h,\mathcal X)
		\lesssim\;&
		\sqrt h\bigl(1+\sqrt d+\sqrt{\log(1/\delta)}\bigr)\\
        &+
		\sqrt{
			\frac{\Gamma h\log(e/h)\mathcal L_{\mathrm{sto}}^h}{\delta}
		}.
	\end{split}
	\]
\end{theorem}

Theorem~\ref{thm:sto_training_error_memorization} provides the stochastic
counterpart of the deterministic overfitting result. If the effective drift
error \(\epsilon_F\) is small along the interpolation path, then the generated
sample remains close to the training set, up to the oracle stochastic
Euler residual and the propagated Gaussian sampling noise. Thus stochastic
overfitting is controlled by three quantities: the stochastic oracle residual,
the accumulated effective training error, and the Gaussian sampling term governed by \(\Sigma_h\).

For the converse direction, define the propagated stochastic training error as
\[
\mathcal L_{\mathrm{sto,prop}}^h
:=
h\sum_{j=1}^K
\Pi_{j-1}^2
\mathbb E_{\rho_{t_j}}
\|\epsilon_F(X,t_j)\|^2.
\]
This quantity measures the effective stochastic training error after accounting
for the contraction weights with which errors at different time steps are
transferred to the final output.

As in the deterministic case, a large propagated stochastic error does not by
itself imply a large endpoint deviation: estimation errors may cancel across
time, and their effect may also be offset by the injected Gaussian perturbation.
We therefore impose the following assumption.
\begin{assumption}
	\label{ass:sto_non_cancellation}
	Write
	\(\mathbb E_{\delta}^{\mathrm{sto}}[\cdot]
	:=\mathbb E[\cdot\mid\mathcal A_h^{\mathrm{sto}}(\delta)]\), let
	\(\nu_{j,\delta}^{\mathrm{sto}}\) be the law of \(Z_j^h\) under this
	conditional distribution, and set \(H_h:=\mathcal E_{F,h}-G_h\).
	There exist constants \(\kappa_F,\kappa_G,\gamma>0\) such that
	\[
	\mathbb E_{\delta}^{\mathrm{sto}}\|H_h\|^2
	\ge
	\kappa_F
	h\sum_{j=1}^K
	\Pi_{j-1}^2
	\mathbb E_{\nu_{j,\delta}^{\mathrm{sto}}}
	\|\epsilon_F(Z,t_j)\|^2
	+
	\kappa_G\operatorname{Tr}(\Sigma_h),
	\]
	and
	\[
	\mathbb E_{\nu_{j,\delta}^{\mathrm{sto}}}
	\|\epsilon_F(Z,t_j)\|^2
	\ge
	\gamma
	\mathbb E_{\rho_{t_j}}
	\|\epsilon_F(X,t_j)\|^2,
	\qquad j=1,\dots,K.
	\]
\end{assumption}

The first inequality prevents strong cancellation between the accumulated
estimation-error shift and the Gaussian sampling perturbation. The second
ensures that the generated trajectory encounters regions where the effective
estimation error is comparable to that under the interpolation distribution.

By unrolling the stochastic Euler--Maruyama recursion with
\(\widehat b=b^*+\epsilon_b\) and \(\widehat s=s^*+\epsilon_s\), we obtain
the decomposition
\[
Z_0^h
=
X_{i_h}
-
\mathcal E_{F,h}
+
G_h
+
R_h.
\]
where \(R_h\) is the non-Gaussian stochastic Euler residual satisfying
\(\|R_h\|\le B_h^{\mathrm{sto}}(\delta)\). This decomposition allows us to
convert a lower bound on the combined propagated perturbation into a lower
bound on the endpoint deviation from the selected training sample.

\begin{figure*}[!t]
    \centering

    \begin{minipage}[b]{0.24\textwidth}
        \centering
        \includegraphics[width=\textwidth]{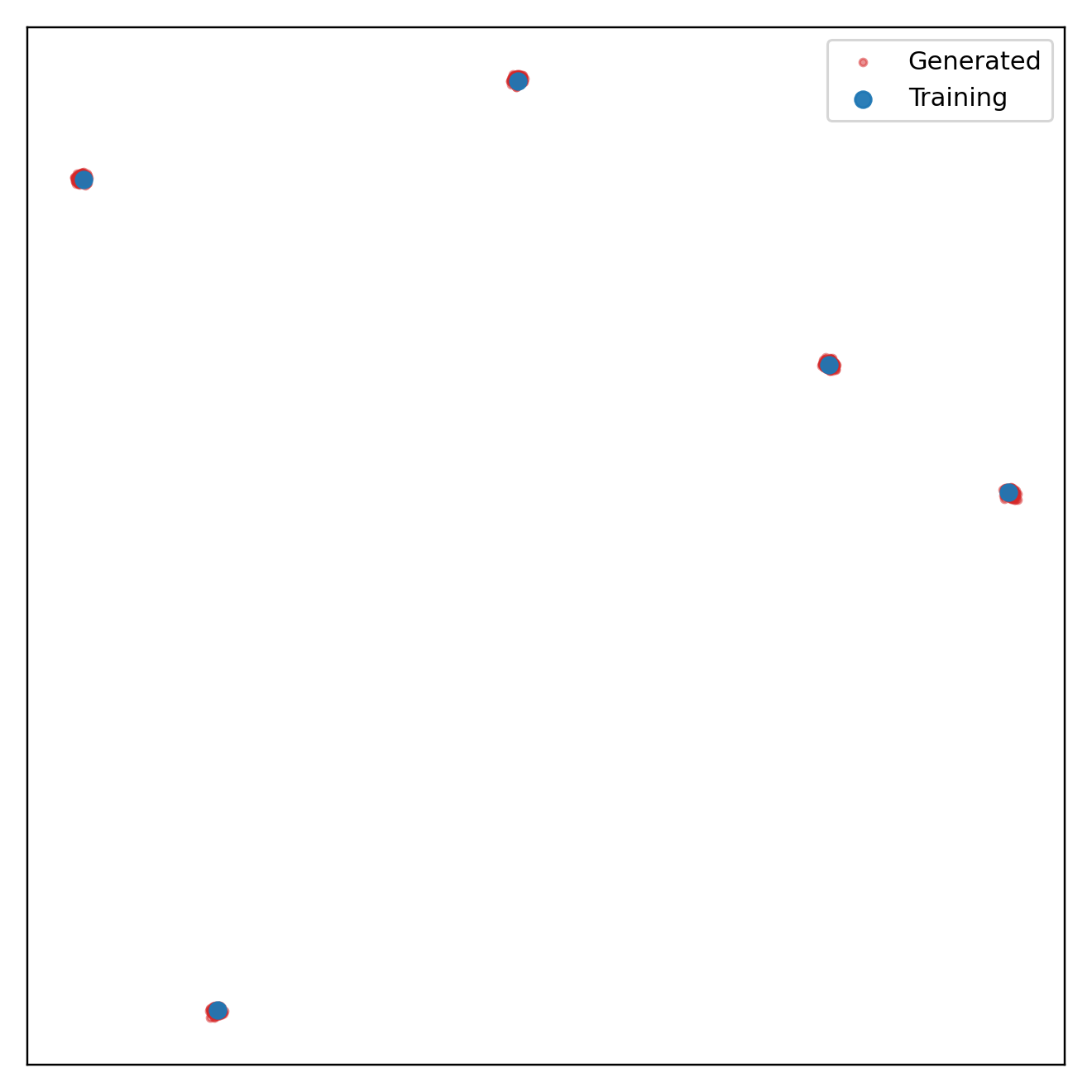}
        {\small (a) Deterministic}
    \end{minipage}
    \hfill
    \begin{minipage}[b]{0.24\textwidth}
        \centering
        \includegraphics[width=\textwidth]{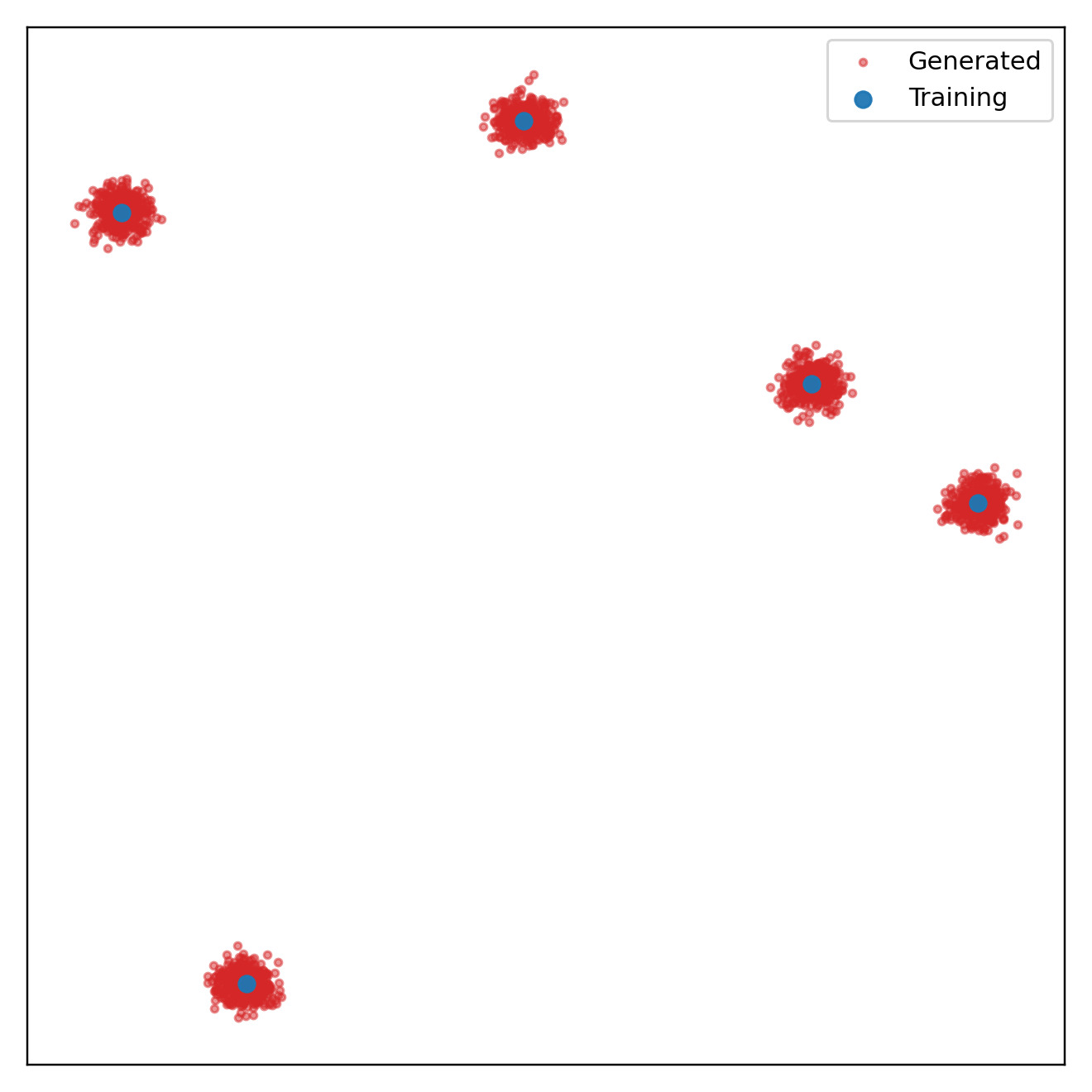}
        {\small (b) Stochastic}
    \end{minipage}
    \hfill
    \begin{minipage}[b]{0.24\textwidth}
        \centering
        \includegraphics[width=\textwidth]{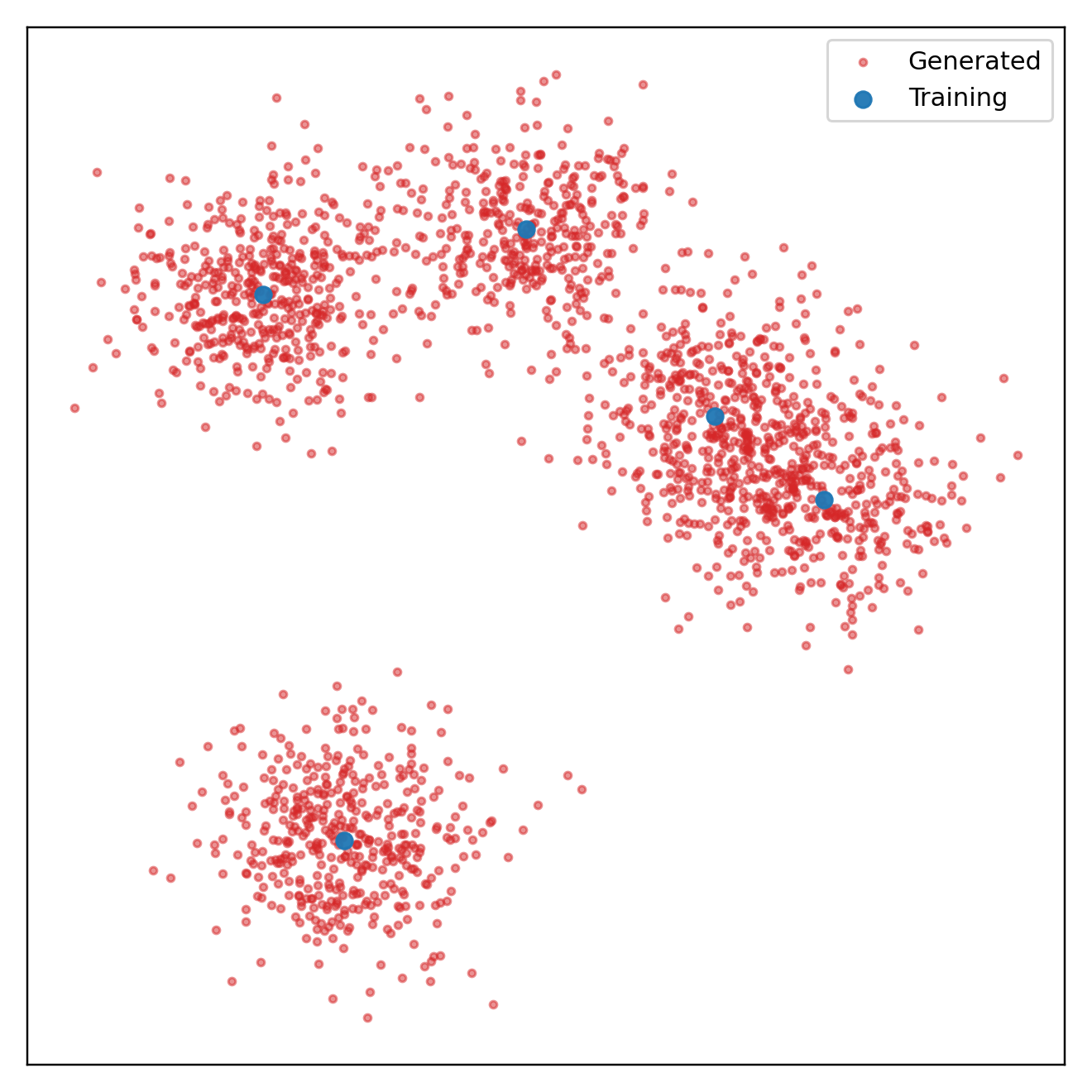}
        {\small (c) Changed \(\sigma\)}
    \end{minipage}
    \hfill
    \begin{minipage}[b]{0.24\textwidth}
        \centering
        \includegraphics[width=\textwidth]{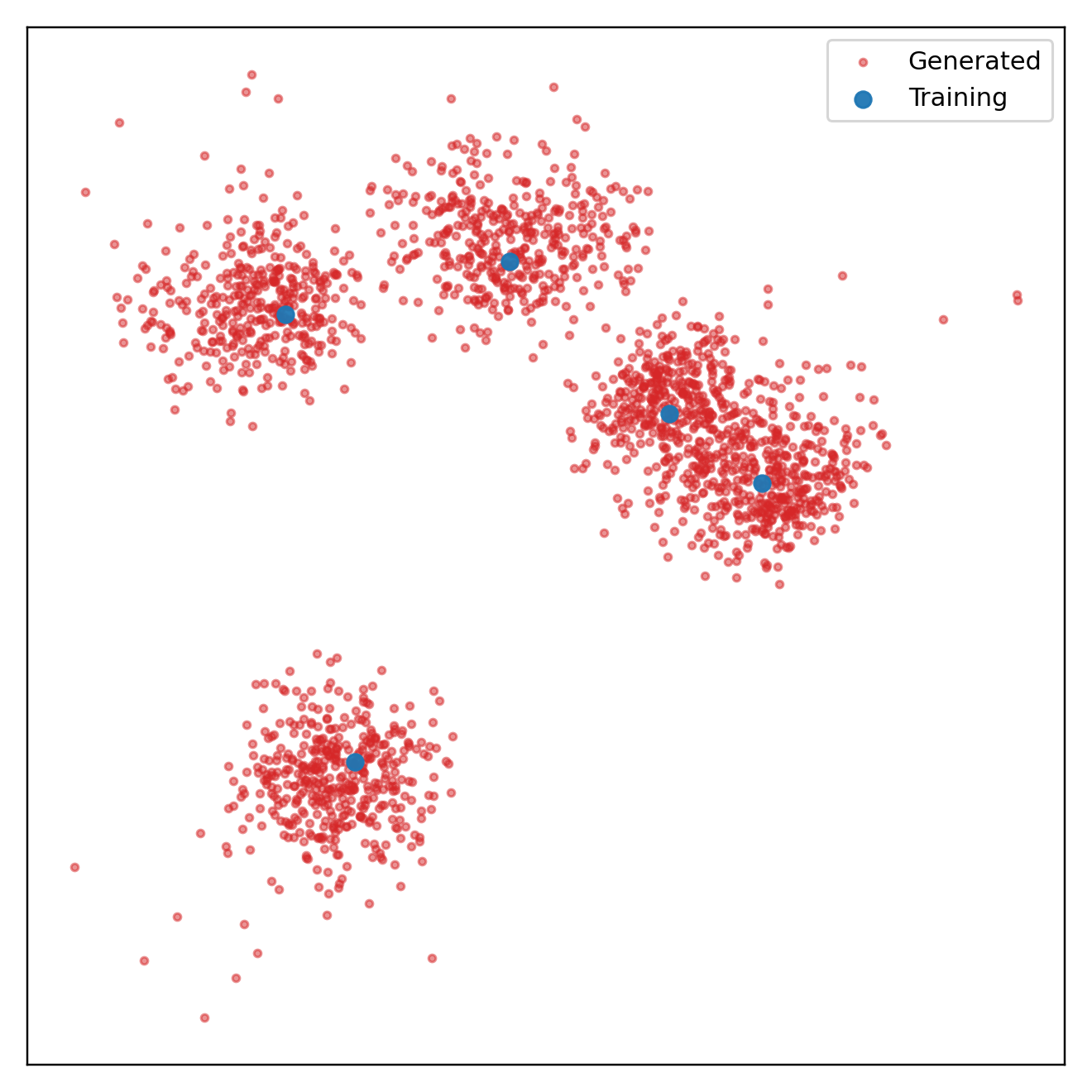}
        {\small (d) Changed \(h\)}
    \end{minipage}

    \caption{Oracle generation. Panel (a) shows deterministic oracle generation. Panel (b) shows stochastic oracle generation. Panels (c)--(d) vary the injected noise scale and Euler step size, respectively.}
    \label{fig:sim_section3}

    \vspace{0.8em}

    \begin{minipage}[b]{0.24\textwidth}
        \centering
        \includegraphics[width=\textwidth]{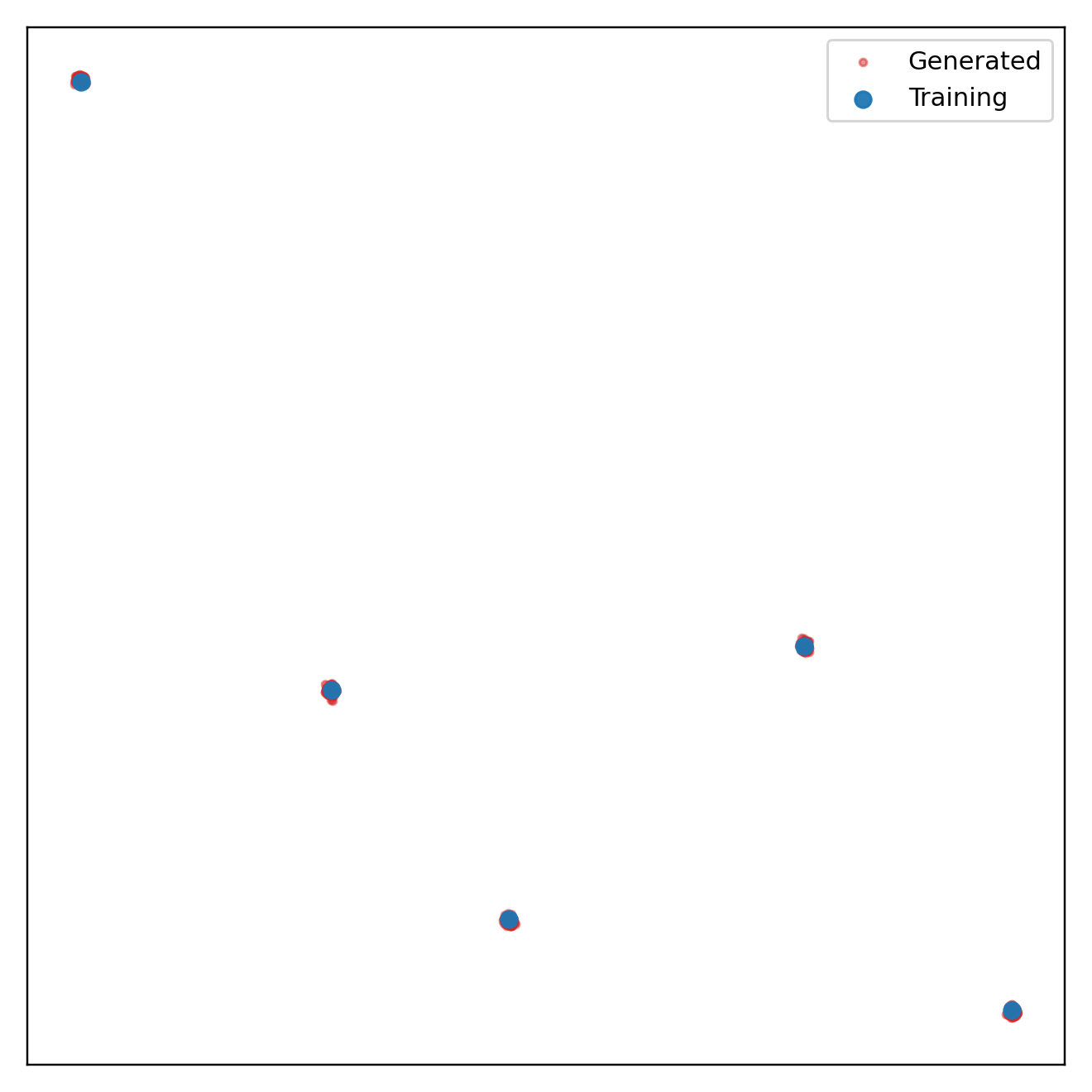}
        {\small (a) Det. overfit}
    \end{minipage}
    \hfill
    \begin{minipage}[b]{0.24\textwidth}
        \centering
        \includegraphics[width=\textwidth]{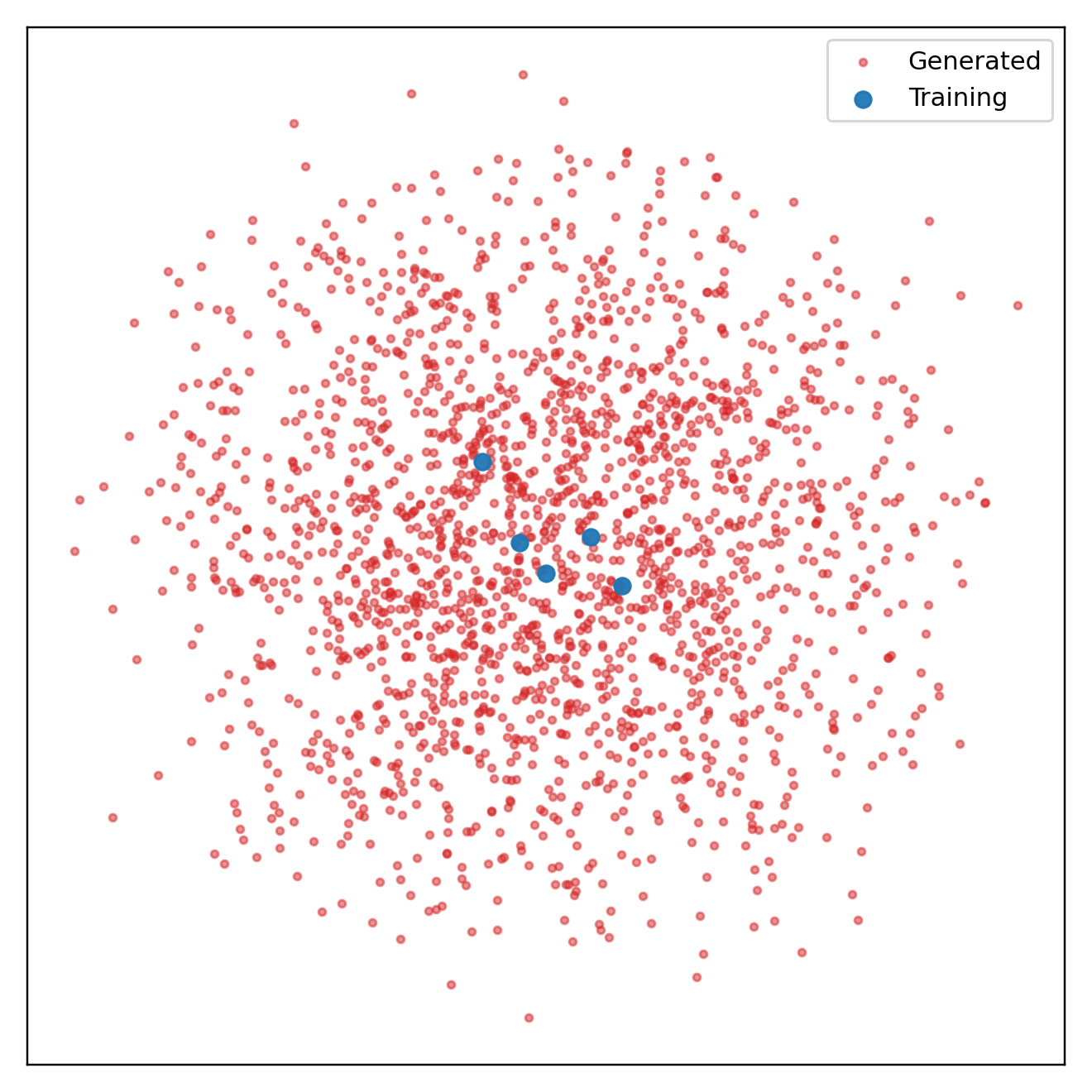}
        {\small (b) Det. underfit}
    \end{minipage}
    \hfill
    \begin{minipage}[b]{0.24\textwidth}
        \centering
        \includegraphics[width=\textwidth]{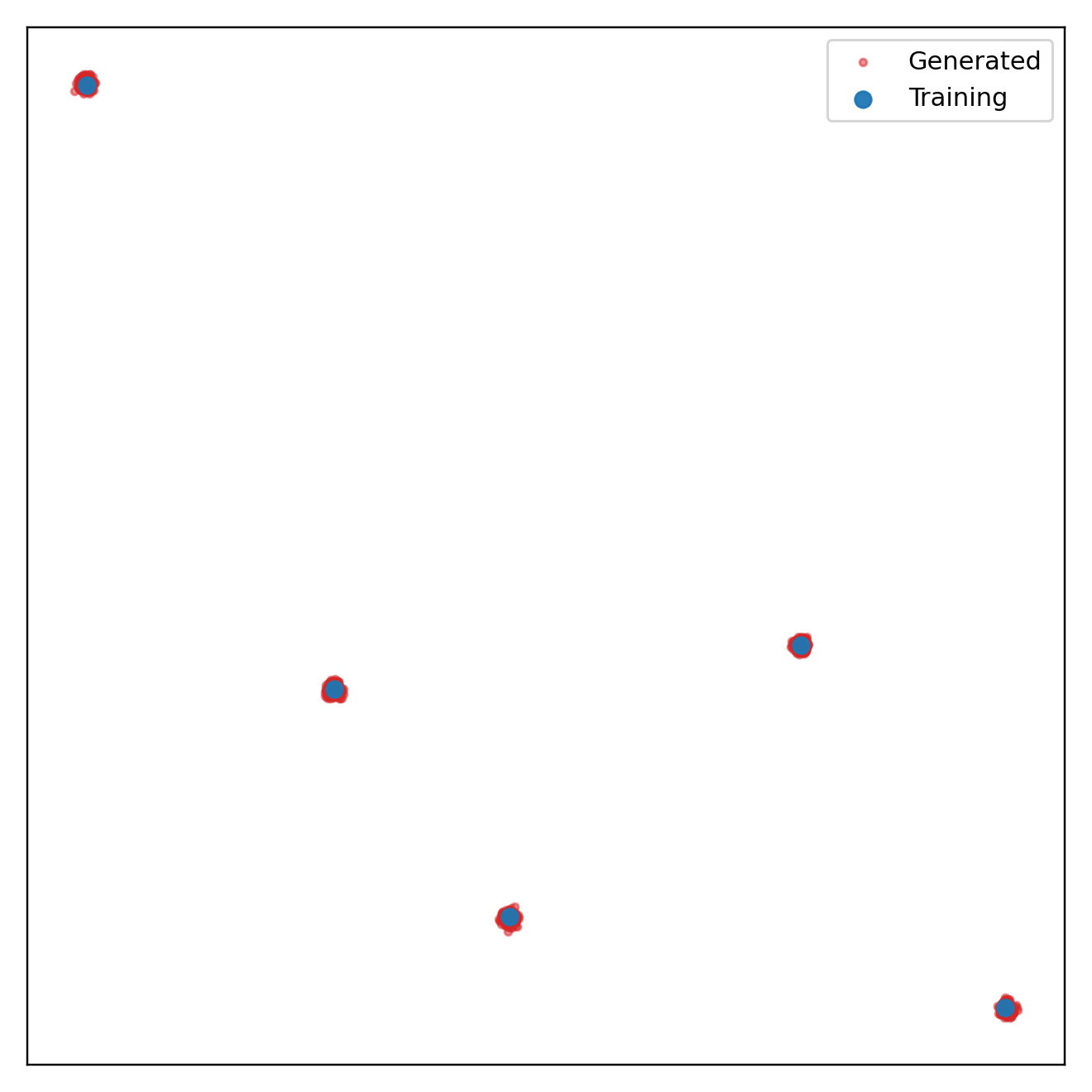}
        {\small (c) Sto. overfit}
    \end{minipage}
    \hfill
    \begin{minipage}[b]{0.24\textwidth}
        \centering
        \includegraphics[width=\textwidth]{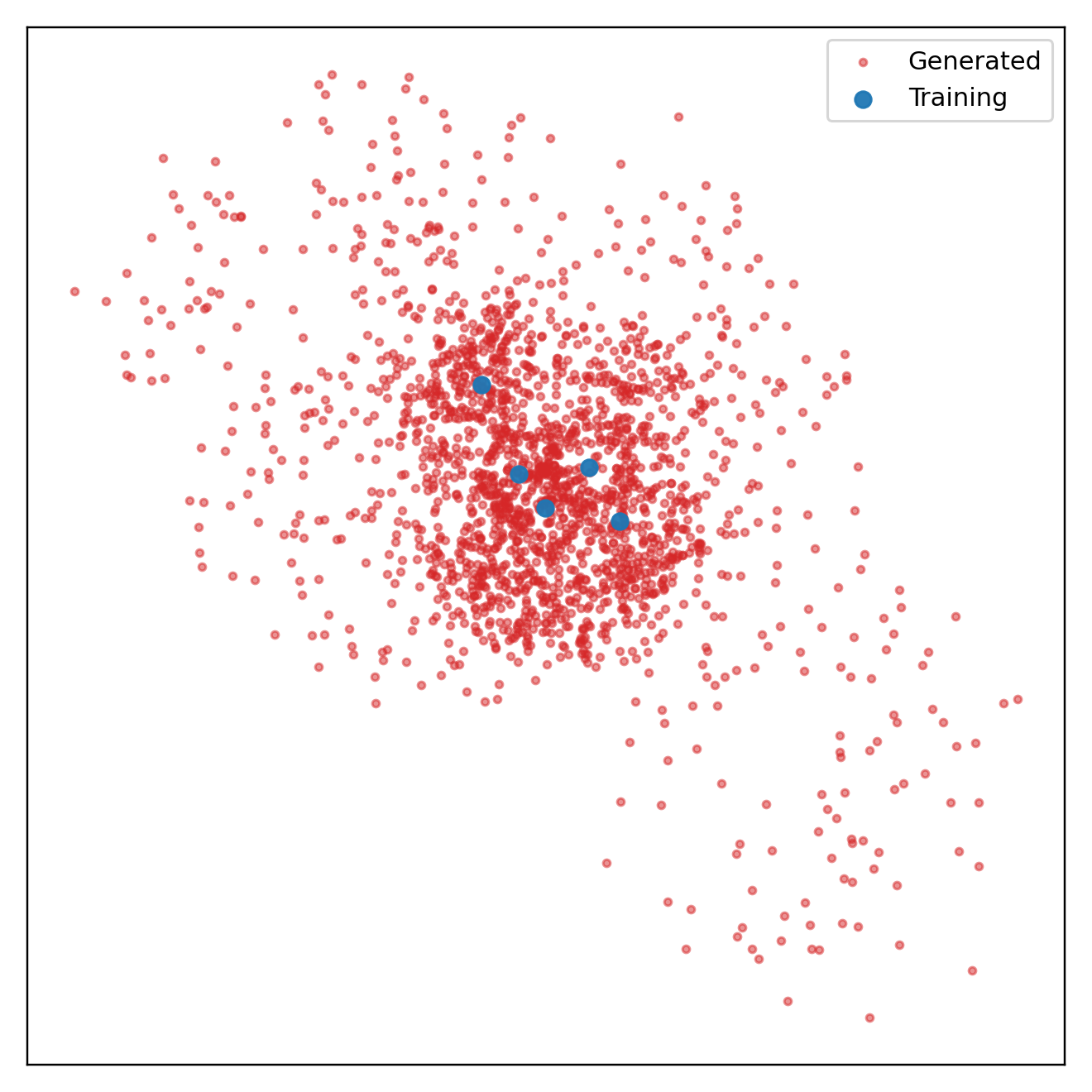}
        {\small (d) Sto. underfit}
    \end{minipage}

    \caption{Generation under estimation error. Panels (a)--(b) illustrate overfitting and underfitting in deterministic generation, respectively, while panels (c)--(d) illustrate the corresponding overfitting and underfitting behaviors in stochastic generation.}
    \label{fig:sim_section4}
\end{figure*}

\begin{theorem}
	\label{thm:sto_training_error_underfitting}
	Under Assumptions~\ref{ass:hp_sto_terminal}, \ref{ass:concentrability}, and~\ref{ass:sto_non_cancellation}, if
	\[
	\kappa_F\gamma\mathcal L_{\mathrm{sto,prop}}^h
	+
	\kappa_G\operatorname{Tr}(\Sigma_h)
	\ge
	\bigl(\tau+B_h^{\mathrm{sto}}(\delta)\bigr)^2,
	\]
	then
	\[
	\left(
	\mathbb E_{\delta}^{\mathrm{sto}}\|Z_0^h-X_{i_h}\|^2
	\right)^{1/2}
	\ge
	\tau.
	\]
\end{theorem}

Theorem~\ref{thm:sto_training_error_underfitting} shows that stochastic
underfitting can arise from two sources: a large effective estimation error and
excessive propagated sampling noise. In either case, if the perturbations do not
cancel, the final sample deviates from the oracle-selected training sample. This
means that the learned stochastic drift fails to reproduce the finite-sample
attractor, even after accounting for the Gaussian sampling perturbation.

\section{Experiments}
\label{sec:simulation_experiments}

We use controlled simulations to verify the theoretical mechanisms developed in
the main text. Additional experiments on real datasets are provided in the supplementary material.

We conduct controlled simulations on two-dimensional synthetic data. In each run, we draw \(n=5\) training samples
uniformly from \([-1,1]^2\), initialize the reverse sampler from
\(\mathcal N(0,I_2)\), and generate 2000 samples. The interpolation schedule is \(\alpha(t)=1-t,\ \beta(t)=t,\ \gamma(t)=\sqrt{t(1-t)},\ \zeta(t)=\sigma\sqrt{t(1-t)}\).
We implement deterministic and stochastic generation using the Euler-type updates in Eq.~\ref{eq:euler_template}.

We first consider oracle generation, where the velocity field and score function are
computed from Eq.~\ref{eq:oracle_b} and Eq.~\ref{eq:oracle_s}. To isolate the effect of stochastic
noise and discretization, we also conduct two single-factor ablations: {one varies the noise scale \(\sigma\) from 0.08 to 0.3, while the other varies the step size}
\(h\) from 0.005 to 0.08. The results are shown in Fig.~\ref{fig:sim_section3}. Panel (a) shows that deterministic oracle samples almost coincide with the training samples.
Panel (b) shows that stochastic generation produces Gaussian-like clouds around
the training samples. Panels (c) and (d) show that the spread of these clouds changes as
\(\sigma\) or \(h\) varies. These observations support the memorization
statement in Theorem~\ref{thm:hp_det_euler} and the stochastic decomposition in
Theorem~\ref{thm:sto_gaussian_error}.

We next study generation in the presence of controlled estimation errors. Specifically, we perturb the oracle velocity field and score function as
\[
\widehat b=b^*+\epsilon_b,\qquad \widehat s=s^*+\epsilon_s.
\]
The error fields are parameterized by scalar magnitudes \(\eta_b\) and
\(\eta_s\):
\[
\epsilon_b(z,t)=\eta_b\,\tilde{\epsilon}_b(z,t),\qquad
\epsilon_s(z,t)=\eta_s\,\tilde{\epsilon}_s(z,t),
\]
where the normalized vector fields satisfy
\[
\|\tilde\epsilon_b\|_\infty \le 1,\qquad
\|\tilde\epsilon_s\|_\infty \le 1.
\]
For deterministic generation, we use \(\eta_b=0.001\) in the overfitting regime
and \(\eta_b=400\) in the underfitting regime. For stochastic generation, we use
\((\eta_b,\eta_s)=(0.00008,0.00004)\) in the overfitting regime and
\((\eta_b,\eta_s)=(30,30)\) in the underfitting regime.

The results are shown in Fig.~\ref{fig:sim_section4}. In the overfitting
regimes, the generated samples remain close to the training samples,
illustrating finite-sample memorization. In contrast, in the underfitting
regimes, the generated samples spread over a much larger region and no longer
represent meaningful samples from the empirical target distribution. These
observations are consistent with the overfitting and underfitting
characterizations in
Theorems~\ref{thm:det_training_error_memorization} -- \ref{thm:sto_training_error_underfitting}.

\section{Conclusion and Future Work}

This paper develops a theoretical framework for analyzing memorization in stochastic interpolation models. Our results show that generated samples can be represented as selected training samples perturbed by a discretization-induced term, an estimation-error-induced term, and, in the stochastic case, Gaussian sampling noise. This characterization further leads to theoretical definitions of overfitting and underfitting in generative models. Controlled simulations provide empirical support for the theoretical findings. An interesting direction for future work is to investigate whether similar memorization mechanisms also arise in other generative modeling paradigms, such as VAEs, GANs, and autoregressive models. 

\bibliography{reference}

\clearpage
\onecolumn
\appendix

\setcounter{proposition}{0}
\setcounter{theorem}{0}
\setcounter{corollary}{0}
\setcounter{lemma}{0}
\setcounter{definition}{0}
\setcounter{assumption}{0}
\setcounter{remark}{0}

\renewcommand{\theproposition}{S\arabic{proposition}}
\renewcommand{\thetheorem}{S\arabic{theorem}}
\renewcommand{\thecorollary}{S\arabic{corollary}}
\renewcommand{\thelemma}{S\arabic{lemma}}
\renewcommand{\thedefinition}{S\arabic{definition}}
\renewcommand{\theassumption}{S\arabic{assumption}}
\renewcommand{\theremark}{S\arabic{remark}}

	\section{Impact of Generated Samples on Downstream Classification}
	\label{sec:classification_appendix}
	
	We evaluate how generated samples affect downstream classification on
	MNIST and FashionMNIST. This experiment is designed to provide empirical support
	for our theoretical characterization of finite-data stochastic generation. In
	particular, our theory suggests that, when the generator is trained on a limited
	empirical distribution, its generated samples tend to behave like training
	samples corrupted by additional noise rather than genuinely independent new
	samples from the population distribution. Therefore, we examine whether samples
	produced by a finite-data generator provide useful data augmentation, or whether
	they mainly act as noisy variants of the training samples.
	
	For a target sample size \(N\), we compare the following five training-set
	constructions. \textbf{Half Samples} uses \(N/2\) real samples. \textbf{Full
		Samples} uses \(N\) real samples. \textbf{Generated Samples} uses \(N/2\) real
	samples together with \(N/2\) generated samples from a generator trained on the
	same \(N/2\) real samples. \textbf{Full Generated Samples} uses generated
	samples from a generator trained on the full training set. \textbf{Noised
		Samples} uses \(N/2\) real samples together with \(N/2\) directly perturbed
	noisy variants. All values in Tables~\ref{tab:mnist_app}
	and~\ref{tab:fashionmnist_app} denote test accuracy (\%).
	
	For both MNIST and FashionMNIST, the generator is implemented as a standard DDPM
	with a U-Net backbone. It is trained using Adam with learning rate \(10^{-3}\),
	batch size \(128\), a linear noise schedule with \(T=500\), and \(25{,}000\)
	training steps. The downstream classifier is a four-layer fully connected
	network with ReLU activations. It is trained using Adam with learning rate
	\(10^{-4}\), cross-entropy loss, batch size \(100\), and up to \(10{,}000\)
	optimization steps. The noise intensity used to construct direct noisy variants
	is \(10^{-5}\).
	
	\begin{table}[ht]
		\centering
		\caption{Classification accuracy on MNIST.}
		\label{tab:mnist_app}
		\begin{tabular}{lccccc}
			\toprule
			Sample Size &
			\makecell{Half\\Samples} &
			\makecell{Noised\\Samples} &
			\makecell{Generated\\Samples} &
			\makecell{Full Generated\\Samples} &
			\makecell{Full\\Samples} \\
			\midrule
			100   & 68.84 & 75.30 & 68.61 & 73.74 & 77.57 \\
			200   & 77.57 & 83.00 & 78.33 & 81.80 & 81.70 \\
			400   & 81.70 & 84.96 & 80.43 & 84.88 & 85.97 \\
			1000  & 87.29 & 88.65 & 86.18 & 88.69 & 89.22 \\
			2000  & 89.22 & 90.44 & 89.78 & 90.48 & 91.30 \\
			4000  & 91.30 & 92.94 & 91.82 & 92.24 & 93.15 \\
			10000 & 93.15 & 94.23 & 93.36 & 93.42 & 94.48 \\
			\bottomrule
		\end{tabular}
	\end{table}
	
	\begin{table}[ht]
		\centering
		\caption{Classification accuracy on FashionMNIST.}
		\label{tab:fashionmnist_app}
		\begin{tabular}{lccccc}
			\toprule
			Sample Size &
			\makecell{Half\\Samples} &
			\makecell{Noised\\Samples} &
			\makecell{Generated\\Samples} &
			\makecell{Full Generated\\Samples} &
			\makecell{Full\\Samples} \\
			\midrule
			100   & 62.26 & 68.69 & 65.89 & 60.80 & 68.32 \\
			200   & 68.32 & 73.78 & 70.78 & 67.81 & 71.04 \\
			400   & 71.04 & 76.18 & 73.07 & 72.50 & 75.77 \\
			1000  & 77.64 & 79.60 & 78.55 & 76.90 & 80.14 \\
			2000  & 80.14 & 82.32 & 81.41 & 80.00 & 81.83 \\
			4000  & 81.83 & 84.21 & 82.65 & 81.50 & 82.64 \\
			10000 & 84.17 & 84.82 & 83.72 & 83.31 & 84.14 \\
			\bottomrule
		\end{tabular}
	\end{table}
	
	The results provide downstream evidence for noisy memorization. Generators
	trained on the full training set can provide useful augmentation, as shown by
	the competitive performance of the Full Generated Samples setting. However,
	generators trained on the same small subset often yield lower classification
	accuracy than simple noisy variants. This suggests that finite-data generation
	may stay close to the empirical training set and introduce less effective
	diversity than direct perturbations. These observations are consistent with our
	theoretical characterization of finite-data stochastic generation as training
	samples corrupted by sampling noise and accumulated approximation errors.

	\newpage
	
	\section{Theoretical Proofs for the Main Text}
	\label{sec:proofs}
	
	This section follows the order of the main theoretical results. We use the
	notation of the main text throughout.
	
	\subsection{Proof of Proposition 1: Oracle Velocity Field}
	
	\begin{proof}
		According to the definition of $b( z,t)$, when $\rho_1$ is the Gaussian distribution, the optimal velocity field can be derived as  
		
		\begin{equation*}
			\begin{split}
				b^*( z,t) &= \mathbb E \Big[\alpha'(t)Z_0+\beta'(t)Z_1+\gamma'(t)\eta|Z_t={z}\Big]\\
				&=\frac{1}{p(Z_t= z)}\int \int\Big[\alpha'(t)z_0+\beta'(t)z_1+\gamma'(t)\eta\Big]\rho_0(z_0)\rho_1(z_1)P_{\eta}\Big(\frac{{z}-\alpha(t)z_0-\beta(t)z_1}{\gamma(t)}\Big)dz_0dz_1\\
				&=\frac{1}{p(Z_t= z)}\int \int\Big\{\big[\alpha'(t)-\frac{\gamma'(t)}{\gamma(t)}\alpha(t)\big]z_0 + \big[\beta'(t)-\frac{\gamma'(t)}{\gamma(t)}\beta(t)\big]z_1 + 
				\frac{\gamma'(t)}{\gamma(t)}{z} \Big\}\\
				&\quad \quad \rho_0(z_0)\rho_1(z_1)P_{\eta}\Big(\frac{{z}-\alpha(t)z_0-\beta(t)z_1}{\gamma(t)}\Big)dz_0dz_1\\
				&=\left\{\int\int\rho_0(z_0)\rho_1(z_1)P_{\eta}\Big(\frac{{z}-\alpha(t)z_0-\beta(t)z_1)}{\gamma(t)}\Big)dz_0dz_1\right\}^{-1}\\
				&\quad\quad\int \int\Big\{\big[\alpha'(t)-\frac{\gamma'(t)}{\gamma(t)}\alpha(t)\big]z_0 + \big[\beta'(t)-\frac{\gamma'(t)}{\gamma(t)}\beta(t)\big]z_1 + 
				\frac{\gamma'(t)}{\gamma(t)}{z} \Big\}\\
				&\quad \quad \rho_0(z_0)\rho_1(z_1)P_{\eta}\Big(\frac{{z}-\alpha(t)z_0-\beta(t)z_1}{\gamma(t)}\Big)dz_0dz_1\\
				&=\left\{\int\int\rho_0(z_0)\exp\Big(-\frac{\Vert z_1-\frac{\beta(t)}{\gamma^2(t)+\beta^2(t)}( z-\alpha(t)z_0)\Vert^2}{2\frac{\gamma^2(t)}{\gamma^2(t)+\beta^2(t)}}\Big)\exp\Big(-\frac{\Vert z-\alpha(t)z_0)\Vert^2}{2(\gamma^2(t)+\beta^2(t))}\Big)dz_1dz_0\right\}^{-1}\\
				&\quad\quad\int \int\Big\{\big[\alpha'(t)-\frac{\gamma'(t)}{\gamma(t)}\alpha(t)\big]z_0 + \big[\beta'(t)-\frac{\gamma'(t)}{\gamma(t)}\beta(t)\big]z_1 + 
				\frac{\gamma'(t)}{\gamma(t)}{z} \Big\}\\
				&\quad \quad \rho_0(z_0)\exp\Big(-\frac{\Vert z_1-\frac{\beta(t)}{\gamma^2(t)+\beta^2(t)}( z-\alpha(t)z_0)\Vert^2}{2\frac{\gamma^2(t)}{\gamma^2(t)+\beta^2(t)}}\Big)\exp\Big(-\frac{\Vert z-\alpha(t)z_0)\Vert^2}{2(\gamma^2(t)+\beta^2(t))}\Big)dz_1dz_0\\
			\end{split}
		\end{equation*}
		
		Treating $z_1$ as a Gaussian random variable, we obtain
		\begin{equation*}
			\begin{split}
				&\rho_1(z_1)P_{\eta}\Big(\frac{ z-\alpha(t)z_0-\beta(t)z_1}{\gamma(t)}\Big)\\
				\propto&\exp\Big(-\frac{\Vert z_1\Vert^2}{2}\Big)\exp\Big(-\frac{\Vert  z-\alpha(t)z_0-\beta(t)z_1\Vert^2}{2\gamma^2(t)}\Big)\\
				\propto&\exp\Big(-\frac{\Vert z_1\Vert^2}{2}\Big)\exp\Big(-\frac{\beta^2(t)\Vert z_1\Vert^2-2\beta(t)z_1^T( z-\alpha(t)z_0)+\Vert z-\alpha(t)z_0\Vert^2}{2\gamma^2(t)}\Big)\\
				\propto&\exp\Big(-\frac{(\gamma^2(t)+\beta^2(t))\Vert z_1\Vert^2-2\beta(t)z_1^T( z-\alpha(t)z_0)+\Vert z-\alpha(t)z_0\Vert^2}{2\gamma^2(t)}\Big)\\
				\propto&\exp\Big(-\frac{\Vert z_1-\frac{\beta(t)}{\gamma^2(t)+\beta^2(t)}( z-\alpha(t)z_0)\Vert^2}{2\frac{\gamma^2(t)}{\gamma^2(t)+\beta^2(t)}}\Big)\exp\Big(-\frac{\Vert z-\alpha(t)z_0)\Vert^2}{2(\gamma^2(t)+\beta^2(t))}\Big).
			\end{split}
		\end{equation*}
		
		Therefore $b^*( z,t)$ has expression that
		\begin{equation*}
			\begin{split}
				b^*( z,t)&= \left\{\int \rho_0(z_0)\exp\Big(-\frac{\Vert{z}-\alpha(t)z_0\Vert^2}{2(\gamma^2(t)+\beta^2(t))}\Big)dz_0\right\}^{-1}\int \Big\{\big[\alpha'(t)-\frac{\gamma'(t)}{\gamma(t)}\alpha(t)\big]z_0+\big[\beta'(t)-\frac{\gamma'(t)}{\gamma(t)}\beta(t)\big]\\
				&\quad\quad\frac{\beta(t)({z}-{\alpha(t)}z_0)}{\gamma^2(t)+\beta^2(t)}+\frac{\gamma'(t)}{\gamma(t)}{z}\Big\}\rho_0(z_0)\exp\big(-\frac{\Vert {z}-\alpha(t)z_0\Vert^2}{2(\gamma^2(t)+\beta^2(t))}\big)dz_0\\
				&=\int\Big[\alpha'(t)-\frac{\gamma'(t)}{\gamma(t)}\alpha(t)-\frac{\beta'(t)\gamma(t)-\gamma'(t)\beta(t)}{\gamma^3(t)+\beta^2(t)\gamma(t)}\alpha(t)\beta(t)\Big]z_0+\Big[\frac{\gamma'(t)}{\gamma(t)}+\frac{\beta'(t)\gamma(t)-\gamma'(t)\beta(t)}{\gamma^3(t)+\beta^2(t)\gamma(t)}\beta(t)\Big]{z}\\
				&\quad\quad\frac{\exp\big(-\frac{\Vert {z}-\alpha(t)z_0\Vert^2}{2(\gamma^2(t)+\beta^2(t))}\big)}{\int \exp\big(-\frac{\Vert {z}-\alpha(t)z_0\Vert^2}{2(\gamma^2(t)+\beta^2(t))}\big)\rho_0(z_0)dz_0}\rho_0(z_0)dz_0.\\
			\end{split}
		\end{equation*}

		Considering that $\rho_0 = \sum_{i=1}^n\delta(X_i)$, the integral can be written as a summation over the samples $\{X_i\}_{i=1}^n$, thus $b^*({z},t)$ has the following expression
		\begin{equation*}
			\begin{split}
				b^*({z},t) &=\sum_{i=1}^n\Big\{\Big[\frac{\gamma(t)\gamma'(t)+\beta'(t)\beta(t)}{\gamma^2(t)+\beta^2(t)}\Big]{z}-\Big[\frac{\gamma(t)\gamma'(t)+\beta'(t)\beta(t)}{\gamma^2(t)+\beta^2(t)}\alpha(t)-\alpha'(t)\Big]X_i\Big\}\frac{\exp\big(-\frac{\Vert {z}-\alpha(t)X_i\Vert^2}{2(\gamma^2(t)+\beta^2(t))}\big)}{\sum_{j=1}^n\exp\big(-\frac{\Vert {z}-\alpha(t)X_j\Vert^2}{2(\gamma^2(t)+\beta^2(t))}\big)}\\
				&:=\sum_{i=1}^n\frac{1}{C_3(t)}\Big[C_1(t){z}-C_2(t)X_i\Big]\frac{\exp\big(-\frac{\Vert {z}-\alpha(t)X_i\Vert^2}{2C_3(t)}\big)}{\sum_{j=1}^n\exp\big(-\frac{\Vert {z}-\alpha(t)X_j\Vert^2}{2C_3(t)}\big)},
			\end{split}
		\end{equation*}
		
		where
		\begin{equation*}
			\begin{split}
				&C_1(t) = \gamma(t)\gamma'(t)+\beta'(t)\beta(t),\\
				&C_2(t)=\Big[\gamma(t)\gamma'(t)+\beta'(t)\beta(t)\Big]\alpha(t)-\Big[\gamma^2(t)+\beta^2(t)\Big]\alpha'(t),\\
				&C_3(t)=\gamma^2(t)+\beta^2(t).
			\end{split}
		\end{equation*}
	\end{proof}
	
	\subsection{Proof of Proposition 1: Oracle Score}
	
	We begin by presenting two related lemmas that establish some useful properties of $Z_t$ with respect to $Z_0$ and $Z_1$.
	\begin{lemma}[Tweedie's Formula]
		Let \( Y \) be a random variable following an exponential family distribution, such that \( Y = \theta + \varepsilon \), where
		\begin{itemize}
			\item \(\theta\) is an unknown parameter follows a prior distribution,
			\item \(\varepsilon\) is independent noise with \(\mathbb{E}[\varepsilon] = 0\) and \(\text{Var }(\varepsilon) = \sigma^2\).
		\end{itemize}
		Then, the posterior expectation of \(\theta\) given \( Y =  y \) is:
		\[
		\mathbb{E}\big[\theta \mid Y = y \big] = y  + \sigma^2 \frac{d}{dy} \log P_Y( y).
		\]
		
		\label{lemma:tweedie}
	\end{lemma}
	
	\begin{lemma} The score function $s({z},t)$ can be expressed as 
		\begin{equation*}
			s({z},t)=-\frac{1}{\gamma(t)}\mathbb{E}\big[\eta|Z_t={z}\big].
		\end{equation*}
		\label{lemma:expression_score}
	\end{lemma}
	\begin{proof}
		Recalling that the definition of $Z_t$ is
		\begin{equation*}
			Z_t = \alpha(t)Z_0 +\beta(t)Z_1 +\gamma(t)\eta.
		\end{equation*}
		Taking conditional expectation with respect to $Z_t$ on both sides, we obtain
		\begin{equation*}
			Z_t = \mathbb E\big[\alpha(t)Z_0+\beta(t)Z_1|Z_t\big] + \gamma(t)\mathbb E\big[\eta|Z_t\big].
		\end{equation*}
		
		Applying Lemma~\ref{lemma:tweedie}, it follows that
		\begin{equation*}
			\mathbb E\big[\alpha(t)Z_0 +\beta(t)Z_1|Z_t={z}\big] = {z} + \gamma^2(t)s({z},t).
		\end{equation*}
		
		Combining the above equation, we have
		\begin{equation}
			s({z},t)=-\frac{1}{\gamma(t)}\mathbb E\big[\eta|Z_t={z}\big].
			\label{s_expression3}
		\end{equation}
		
	\end{proof}

	Next we present the proof of Proposition~1.
	
	\begin{proof}
		The velocity field $b({z},t)$ can be further expressed as
		\begin{equation}
			\begin{split}
				b({z},t)&=\mathbb E\big[\alpha'(t)Z_0+\beta'(t)Z_1+\gamma'(t)\eta|Z_t={z}\big]\\
				&=\alpha'(t)\mathbb E\big[Z_0|Z_t={z}\big] +\beta'(t)\mathbb  E\big[Z_1|Z_t={z}\big]+\gamma'(t)\mathbb E\big[\eta|Z_t={z}\big].
			\end{split}
			\label{s_expression1}
		\end{equation}
		
		Applying Lemma~\ref{lemma:tweedie}, it follows that
		\begin{equation}
			\alpha(t)\mathbb E\big[Z_0|Z_t={z}\big]={z}+(\beta^2(t)+\gamma^2(t))s({z},t).
			\label{s_expression2}
		\end{equation}
		
		In addition, by Lemma~\ref{lemma:expression_score}, we have
		\begin{equation}
			\begin{split}
				Z_t &= \mathbb E\big[\alpha(t)Z_0 +\beta(t)Z_1|Z_t\big] +\gamma(t)\mathbb E\big[\eta|Z_t\big]\\
				&=\alpha(t)\mathbb E\big[Z_0|Z_t\big] +\beta(t)\mathbb E\big[Z_1|Z_t\big] +\gamma(t)\mathbb E\big[\eta|Z_t\big].
			\end{split}
			\label{s_expression4}
		\end{equation}
		
		Combining Eq.~\ref{s_expression3}, Eq.~\ref{s_expression1}, Eq.~\ref{s_expression2}, and Eq.~\ref{s_expression4}, we arrive at the closed-form expression for the score function
		\begin{equation*}
			s({z},t) = \frac{\alpha(t)}{B(t)}b({z},t)-\frac{\alpha'(t)}{B(t)} z,
		\end{equation*}
		where
		\begin{equation*}
			B(t) = \beta(t)\big[\alpha'(t)\beta(t)-\alpha(t)\beta'(t)\big]+\gamma(t)\big[\gamma(t)\alpha'(t)-\gamma'(t)\alpha(t)\big].
		\end{equation*}
		Thus we have 
		\begin{equation*}
			s^*({z},t) = \frac{\alpha(t)}{B(t)}b^*({z},t)-\frac{\alpha'(t)}{B(t)}z.
		\end{equation*}
		
	\end{proof}
	
	\subsection{Proof of Theorem 1: Deterministic Generation}
	
	\begin{proof}
		Let us define $A(t)=\exp\big(-\int_t^1\frac{C_1(u)}{C_3(u)}du\big)$, $X=[X_1,\dots X_n]$ and introduce the weight function
		\begin{equation*}
			\omega({z},t)= \text{softmax}\Big[\frac{\Vert {z}-\alpha(t)X_1\Vert^2}{2(\gamma^2(t)+\beta^2(t))},\cdots,\frac{\Vert {z}-\alpha(t)X_n\Vert^2}{2(\gamma^2(t)+\beta^2(t))}\Big]^T.
		\end{equation*}
		Let $Z_t = A(t)\kappa(t)$, where $\kappa(t)$ is a time-dependent auxiliary function. Then, by differentiating $\kappa(t)$, we have
		\begin{equation*}
			\begin{split}
				\frac{d\kappa(t)}{dt}&=\frac{A(t)\frac{dZ_t(x)}{dt}-Z_tA'(t)}{A^2(t)}\\
				&=\frac{A(t)\frac{C_1(t)}{C_3(t)}Z_t-A(t)\frac{C_2(t)}{C_3(t)}\cdot X\cdot\omega\big(Z_t,t\big)-\frac{C_1(t)}{C_3(t)}A_tZ_t}{A^2(t)}\\
				&=-\frac{C_2(t)}{C_3(t)A(t)}\cdot X\cdot \omega\big(Z_t,t\big).
			\end{split}
		\end{equation*}
		Let $Z_t(x)$ denote the random variable at time $t$ of the process initialized at point $x$. Then, $Z_t(x)$ can be expressed as follows
		\begin{equation*}
			\begin{split}
				Z_t(x) &= A(t)\kappa(t)\\
				&=\exp\Big(-\int_{t}^1\frac{C_1(u)}{C_3(u)}du\Big)\Big(x+\int_t^1\frac{C_2(u)}{C_3(u)A(u)}\cdot X\cdot \omega\big(Z_u(x),u\big)du\Big).
			\end{split}
		\end{equation*}
		
		Therefore, at $t = 0$, there exists $i\in\{1\cdots n\}$ that the generation result $Z_0(x)$ takes the following value
		\begin{equation}
			\begin{split}
				Z_0(x) &= \lim_{t\rightarrow0}\exp\Big(-\int_t^1\frac{C_1(u)}{C_3(u)}du\Big)\Big(x+\int_{t}^1\frac{C_2(u)}{C_3(u)A(u)}\cdot X\cdot \omega\big(Z_u(x),u\big)du\Big)\\
				&=\lim_{t\rightarrow0}\frac{\int_t^1\frac{C_2(u)}{C_3(u)A(u)}\cdot X\cdot \omega\big(Z_u(x),u\big)du}{\frac{1}{A(t)}}\\
				&=\lim_{t\rightarrow0}\frac{\frac{C_2(t)}{C_3(t)A(t)}\cdot X\cdot \omega\big(Z_t(x),t\big)}{\frac{A'(t)}{A^2(t)}}\\
				&=\lim_{t\rightarrow0}\frac{C_2(t)}{C_1(t)} \cdot X \cdot \omega\big(Z_t(x),t\big)\\
				&=X\cdot\lim_{t\rightarrow0}  \omega\big(Z_t(x),t\big)\\
				&=X_i.
			\end{split}
			\label{eq:oracle_deter_gen}
		\end{equation}
		The sixth equation holds because as $t$ approaches $0$, due to the nature of the softmax function, one element in $\omega$ tends to $1$, while the other elements tend to $0$. Therefore, there exists $i\in\{1,\dots,n\}$ such that the final result will be one of the $X_i$.
		
	\end{proof}
	
	\subsection{Proof of Theorem 1: Stochastic Generation}
	
	\begin{proof}
		The proof has two parts. First, we show that the oracle stochastic generation
		process has the same time marginals as the stochastic interpolation. Second, we
		compute this marginal explicitly under the empirical training distribution.
		
		Let \(q_t\) denote the density of the oracle stochastic generation process at
		time \(t\). The interpolation density \(\rho_t\) associated with the oracle
		velocity field satisfies the continuity equation
		\begin{equation}
			\partial_t\rho_t
			=
			-\nabla\cdot\bigl(b^*(\cdot,t)\rho_t\bigr).
			\label{eq:app_continuity}
		\end{equation}
		This follows from the definition of \(b^*\) as the conditional mean velocity of
		the interpolation path. Moreover, by definition of the oracle score,
		\[
		s^*(z,t)=\nabla_z\log\rho_t(z),
		\]
		and therefore
		\begin{equation}
			s^*(z,t)\rho_t(z)=\nabla\rho_t(z).
			\label{eq:app_score_identity}
		\end{equation}
		
		The oracle stochastic generation dynamics uses the drift \(b^*-\zeta s^*\) and
		diffusion coefficient \(\sqrt{2\zeta}\). Its Fokker--Planck equation, written
		in the same time orientation as the interpolation parameter, is
		\begin{equation}
			\begin{split}
				\partial_t q_t
				=
				&-\nabla\cdot\bigl(b^*(\cdot,t)q_t\bigr)
				+
				\zeta(t)\nabla\cdot\bigl(s^*(\cdot,t)q_t\bigr)
				-
				\zeta(t)\Delta q_t .
			\end{split}
			\label{eq:app_fp}
		\end{equation}
		We now verify that \(q_t=\rho_t\) solves this equation. Substituting
		\(q_t=\rho_t\) into the right-hand side of
		Eq. \ref{eq:app_fp} gives
		\[
		-\nabla\cdot\bigl(b^*(\cdot,t)\rho_t\bigr)
		+
		\zeta(t)\nabla\cdot\bigl(s^*(\cdot,t)\rho_t\bigr)
		-
		\zeta(t)\Delta\rho_t .
		\]
		Using Eq. \ref{eq:app_score_identity},
		\[
		\nabla\cdot\bigl(s^*(\cdot,t)\rho_t\bigr)
		=
		\nabla\cdot(\nabla\rho_t)
		=
		\Delta\rho_t.
		\]
		Hence the two \(\zeta(t)\)-terms cancel exactly, leaving
		\[
		-\nabla\cdot\bigl(b^*(\cdot,t)\rho_t\bigr),
		\]
		which is \(\partial_t\rho_t\) by Eq. (1). Thus
		\(\rho_t\) satisfies the same Fokker--Planck equation as the oracle stochastic
		generation marginal.
		
		The generation process is initialized at \(t=1\) with \(Z_1\sim\rho_1\), and
		the stochastic interpolation also satisfies \(\rho_t|_{t=1}=\rho_1\). Assuming
		the standard uniqueness of solutions to the Fokker--Planck equation for this
		smooth positive density on every compact subinterval of \((0,1)\), we conclude
		that
		\[
		q_t=\rho_t,\qquad 0<t<1.
		\]
		In particular,
		\[
		Z_\varepsilon\sim\rho_\varepsilon
		\]
		for every \(\varepsilon\in(0,1)\).
		
		It remains to compute \(\rho_\varepsilon\). Since
		\[
		\rho_0=\frac1n\sum_{i=1}^n\delta_{X_i},
		\]
		we may write \(Z_0=X_i\), where
		\[
		i\sim\mathrm{Unif}\{1,\dots,n\}.
		\]
		The stochastic interpolation at time \(\varepsilon\) is
		\[
		Z_\varepsilon
		=
		\alpha(\varepsilon)X_i
		+
		\beta(\varepsilon)Z_1
		+
		\gamma(\varepsilon)\eta,
		\]
		where \(Z_1\sim\mathcal N(0,I_d)\), \(\eta\sim\mathcal N(0,I_d)\), and
		\((i,Z_1,\eta)\) are mutually independent. Conditional on the selected index \(i\), the first
		term is deterministic and the remaining two terms are independent Gaussian
		vectors. Therefore
		\[
		\beta(\varepsilon)Z_1+\gamma(\varepsilon)\eta
		\sim
		\mathcal N\bigl(0,
		[\beta^2(\varepsilon)+\gamma^2(\varepsilon)]I_d\bigr)
		=
		\mathcal N(0,C_3(\varepsilon)I_d).
		\]
		Consequently,
		\[
		Z_\varepsilon\mid i
		\sim
		\mathcal N\bigl(
		\alpha(\varepsilon)X_i,
		C_3(\varepsilon)I_d
		\bigr).
		\]
		Averaging over the uniform index \(i\) yields
		\[
		\rho_\varepsilon
		=
		\frac1n\sum_{i=1}^n
		\mathcal N\bigl(
		\alpha(\varepsilon)X_i,
		C_3(\varepsilon)I_d
		\bigr).
		\]
		Equivalently, if
		\[
		\xi_\varepsilon
		:=
		\beta(\varepsilon)Z_1+\gamma(\varepsilon)\eta,
		\]
		then \(\xi_\varepsilon\sim\mathcal N(0,C_3(\varepsilon)I_d)\) and
		\(\xi_\varepsilon\) is independent of \(i\). Hence
		\[
		Z_\varepsilon
		\stackrel{d}{=}
		\alpha(\varepsilon)X_i+\xi_\varepsilon,
		\]
		which proves the claimed finite-time noisy memorization representation.
	\end{proof}

	\subsection{Proof of Theorem 2}
	We first use a lemma to control the distance between the average sample and each individual sample.
	\begin{lemma}
		For an arbitrary point \(z\), define
		\[
		d_{i,k}^h(z):=\|z-\alpha(t_k)X_i\|^2,
		\qquad
		w_{i,k}^h(z):=
		\frac{
			\exp\{-d_{i,k}^h(z)/(2C_3(t_k))\}
		}{
			\sum_{\ell=1}^n
			\exp\{-d_{\ell,k}^h(z)/(2C_3(t_k))\}
		},
		\]
		and
		\[
		\bar X_k^h(z):=\sum_{i=1}^n w_{i,k}^h(z)X_i.
		\]
		When \(z=Z_k^h\), these coincide with the main-text definitions of
		\(w_{i,k}^h\) and \(\bar X_k^h\). For any \(u_k>0\), if
		\(m_k(z)\ge u_k\), then
		\[
		\|\bar X_k^h(z)-X_{i_k(z)}\|
		\le
		D_X(n-1)
		\exp\left(
		-\frac{u_k}{2C_3(t_k)}
		\right).
		\]
	\end{lemma}
	
	\begin{proof}
		Let \(i=i_k(z)\). If \(m_k(z)\ge u_k\), then for each \(j\ne i\),
		\[
		\frac{w_{j,k}^h(z)}{w_{i,k}^h(z)}
		=
		\exp\left(
		-\frac{\|z-\alpha(t_k)X_j\|^2-\|z-\alpha(t_k)X_i\|^2}
		{2C_3(t_k)}
		\right)
		\le
		\exp\left(-\frac{u_k}{2C_3(t_k)}\right).
		\]
		Since \(w_{i,k}^h(z)\le 1\), it follows that
		\[
		\sum_{j\ne i}w_{j,k}^h(z)
		\le
		(n-1)\exp\left(-\frac{u_k}{2C_3(t_k)}\right).
		\]
		Since \(\sum_j w_{j,k}^h(z)(X_j-X_i)\) has no contribution from \(j=i\),
		\[
		\begin{split}
			\|\bar X_k^h(z)-X_i\|
			&\le
			\sum_{j\ne i}w_{j,k}^h(z)\|X_j-X_i\|\\
			&\le
			D_X(n-1)\exp\left(-\frac{u_k}{2C_3(t_k)}\right).
		\end{split}
		\]
	\end{proof}
	
	Next we give the proof of Theorem 2.
	\begin{proof}
		On \(\mathcal A_h(\delta)\), we have
		\[
		m_k(Z_k^h)\ge u_k.
		\]
		Hence, by Lemma 1,
		\[
		\|\bar X_k^h-X_{i_k}\|
		\le
		D_X(n-1)
		\exp\left(
		-\frac{u_k}{2C_3(t_k)}
		\right).
		\]
		For notational simplicity, define $\varepsilon_k^{\mathrm{sm}}:=D_X(n-1)\exp\left(-\frac{u_k}{2C_3(t_k)}\right).$
		
		Let
		\[
		D_k:=\min_{1\le i\le n}\|Z_k^h-\alpha(t_k)X_i\|.
		\]
		Since \(\alpha(t_0)=1\), we have
		\[
		D_0=\operatorname{dist}(Z_0^h,\mathcal X).
		\]
		We first derive a one-step recursion for \(D_{k-1}\). Since
		\(\alpha(t_{k-1})X_{i_k}\) is one of the scaled centers at time
		\(t_{k-1}\), we have
		\[
		D_{k-1}
		\le
		\|Z_{k-1}^h-\alpha(t_{k-1})X_{i_k}\|.
		\]
		Using the Euler update,
		\[
		Z_{k-1}^h
		=
		\lambda_k Z_k^h+c_k\bar X_k^h,
		\]
		we obtain
		\[
		\begin{split}
			Z_{k-1}^h-\alpha(t_{k-1})X_{i_k}
			&=
			\lambda_k Z_k^h+c_k\bar X_k^h-\alpha(t_{k-1})X_{i_k}  \\
			&=
			\lambda_k(Z_k^h-\alpha(t_k)X_{i_k})
			+
			c_k(\bar X_k^h-X_{i_k})
			+
			\bigl(\lambda_k\alpha(t_k)+c_k-\alpha(t_{k-1})\bigr)X_{i_k}.
		\end{split}
		\]
		Taking norms gives
		\[
		\begin{split}
			D_{k-1}
			&\le
			|\lambda_k|\|Z_k^h-\alpha(t_k)X_{i_k}\|
			+
			|c_k|\|\bar X_k^h-X_{i_k}\|
			+
			\left|\lambda_k\alpha(t_k)+c_k-\alpha(t_{k-1})\right|
			\|X_{i_k}\|  \\
			&\le
			|\lambda_k|D_k
			+
			|c_k|\varepsilon_k^{\mathrm{sm}}
			+
			\left|\lambda_k\alpha(t_k)+c_k-\alpha(t_{k-1})\right|M_X.
		\end{split}
		\]
		Define
		\[
		b_k
		:=
		\left|\lambda_k\alpha(t_k)+c_k-\alpha(t_{k-1})\right|M_X
		+
		|c_k|D_X(n-1)
		\exp\left(
		-\frac{u_k}{2C_3(t_k)}
		\right).
		\]
		Then the one-step recursion becomes
		\[
		D_{k-1}\le |\lambda_k|D_k+b_k.
		\]
		
		We now unfold this recursion from \(k=K\) down to \(k=1\). First,
		\[
		D_0\le |\lambda_1|D_1+b_1.
		\]
		Continuing this expansion yields
		\[
		D_0
		\le
		\left(\prod_{k=1}^K|\lambda_k|\right)D_K
		+
		\sum_{j=1}^K
		\left(
		\prod_{\ell=1}^{j-1}|\lambda_\ell|
		\right)b_j,
		\]
		where the empty product is understood as \(1\).
		
		It remains to bound \(D_K\). On \(\mathcal A_h(\delta)\),
		\[
		\|Z_K^h\|\le R_\delta.
		\]
		Since
		\[
		D_K
		=
		\min_{1\le i\le n}\|Z_K^h-\alpha(t_K)X_i\|
		\le
		\|Z_K^h-\alpha(t_K)X_i\|
		\]
		for any \(i\), we can choose an arbitrary training sample and obtain
		\[
		D_K
		\le
		\|Z_K^h\|+|\alpha(t_K)|M_X
		\le
		R_\delta+M_X.
		\]
		Therefore, on \(\mathcal A_h(\delta)\),
		\[
		\begin{split}
			D_0
			\le\;&
			\left(
			\prod_{k=1}^K|\lambda_k|
			\right)
			(R_\delta+M_X)+
			\sum_{j=1}^K
			\left(
			\prod_{\ell=1}^{j-1}|\lambda_\ell|
			\right)
			\Bigg[
			\left|\lambda_j\alpha(t_j)+c_j-\alpha(t_{j-1})\right|M_X\\
			&\hspace{3.5cm}
			+
			|c_j|D_X(n-1)
			\exp\left(
			-\frac{u_j}{2C_3(t_j)}
			\right)
			\Bigg].
		\end{split}
		\]
		By definition, the right-hand side is exactly
		\(B_h^{\mathrm{det}}(\delta)\). Hence, on \(\mathcal A_h(\delta)\),
		\[
		\operatorname{dist}(Z_0^h,\mathcal X)
		=
		D_0
		\le
		B_h^{\mathrm{det}}(\delta).
		\]
		
		Finally, since
		\[
		\mathbb P(\mathcal A_h(\delta))\ge1-\delta,
		\]
		we conclude that
		\[
		\mathbb P\left(
		\operatorname{dist}(Z_0^h,\mathcal X)
		\le
		B_h^{\mathrm{det}}(\delta)
		\right)
		\ge
		1-\delta.
		\]
		This completes the proof.
	\end{proof}
	
	\subsection{Proof of Corollary 1}
	
	\begin{proof}
		For the schedule
		\[
		\alpha(t)=1-t,\qquad
		\beta(t)=t,\qquad
		\gamma(t)=\sqrt{t(1-t)},
		\]
		direct calculation gives
		\[
		C_1(t)=\frac12,\qquad
		C_2(t)=\frac{1+t}{2},\qquad
		C_3(t)=t.
		\]
		Thus, for \(t_k=kh\),
		\[
		a_k=\frac{1}{2k},\qquad
		c_k=\frac{1+t_k}{2k},\qquad
		\lambda_k=1-\frac{1}{2k}.
		\]
		Moreover,
		\[
		\lambda_k\alpha(t_k)+c_k
		=
		\left(1-\frac{1}{2k}\right)(1-kh)
		+
		\frac{1+kh}{2k}
		=
		1-(k-1)h
		=
		\alpha(t_{k-1}),
		\]
		so \(B_{h,\mathrm{sch}}^{\mathrm{det}}=0\).
		
		It remains to bound the terminal-entry and softmax-selection terms. Let
		\[
		P_m:=\prod_{\ell=1}^m\left(1-\frac{1}{2\ell}\right).
		\]
		Since \(\log(1-x)\le -x\) for \(x\in(0,1)\),
		\[
		P_m\le C m^{-1/2}.
		\]
		In particular, \(P_K\le C\sqrt h\), and hence
		\[
		B_{h,\mathrm{init}}^{\mathrm{det}}(\delta)
		\le
		C\sqrt h(1+R_\delta).
		\]
		
		For the softmax-selection term, use \(|c_j|\le C/j\) and
		\(P_{j-1}\le Cj^{-1/2}\). If \(t_j\le\theta\), then the assumed margin
		condition gives
		\[
		\exp\left(-\frac{u_j}{2C_3(t_j)}\right)
		=
		\exp\left(-\frac{u_j}{2t_j}\right)
		\le
		\sqrt h.
		\]
		Therefore the contribution from \(t_j\le\theta\) is bounded by
		\[
		C\sqrt h\sum_{j=1}^K j^{-3/2}
		\le
		C\sqrt h.
		\]
		For the remaining indices \(t_j>\theta\), we only use the trivial bound
		\(\exp[-u_j/(2t_j)]\le1\). Since \(j>\theta/h\),
		\[
		\sum_{t_j>\theta} P_{j-1}|c_j|
		\le
		C\sum_{j>\theta/h} j^{-3/2}
		\le
		C\sqrt h.
		\]
		Thus
		\[
		B_{h,\mathrm{sm}}^{\mathrm{det}}\le C\sqrt h.
		\]
		
		Combining the three bounds gives
		\[
		B_h^{\mathrm{det}}(\delta)
		\le
		C\sqrt h(1+R_\delta).
		\]
		The probability statement follows from Theorem 2. Finally, since
		\(Z_K^h\sim\mathcal N(0,I_d)\), taking
		\[
		R_\delta=\sqrt d+\sqrt{2\log(1/\delta)}
		\]
		gives the displayed high-probability bound in the corollary.
	\end{proof}
	
	\subsection{Proof of Theorem 3}
	
	\begin{proof}
		On \(\mathcal A_h^{\mathrm{sto}}(\delta)\), there exists a random index
		\(i_h\in\{1,\dots,n\}\) such that
		\[
		\|Z_K^h\|\le R_\delta,
		\qquad
		i_k(Z_k^h)=i_h,
		\qquad
		m_k(Z_k^h)\ge u_k,
		\quad k=1,\dots,K.
		\]
		By Lemma 1, for every \(k\),
		\[
		\|\bar X_k^h-X_{i_h}\|
		\le
		D_X(n-1)
		\exp\left(
		-\frac{u_k}{2C_3(t_k)}
		\right).
		\]
		For brevity, define the softmax residual
		\[
		\varepsilon_k^{\mathrm{sm}}
		:=
		D_X(n-1)
		\exp\left(
		-\frac{u_k}{2C_3(t_k)}
		\right).
		\]
		Thus, on \(\mathcal A_h^{\mathrm{sto}}(\delta)\),
		\[
		\|\bar X_k^h-X_{i_h}\|
		\le
		\varepsilon_k^{\mathrm{sm}}.
		\]
		
		The stochastic Euler update is
		\[
		Z_{k-1}^h
		=
		r_kZ_k^h+q_k\bar X_k^h+\sigma_k\xi_k.
		\]
		Define the error relative to the selected training sample:
		\[
		e_k:=Z_k^h-X_{i_h}.
		\]
		Subtracting \(X_{i_h}\) from both sides of the update gives
		\[
		\begin{split}
			e_{k-1}
			&=
			r_kZ_k^h+q_k\bar X_k^h+\sigma_k\xi_k-X_{i_h}  \\
			&=
			r_k(Z_k^h-X_{i_h})
			+
			q_k(\bar X_k^h-X_{i_h})
			+
			(r_k+q_k-1)X_{i_h}
			+
			\sigma_k\xi_k .
		\end{split}
		\]
		Hence
		\[
		e_{k-1}
		=
		r_ke_k
		+
		d_k^h
		+
		\sigma_k\xi_k,
		\]
		where
		\[
		d_k^h
		:=
		q_k(\bar X_k^h-X_{i_h})
		+
		(r_k+q_k-1)X_{i_h}.
		\]
		
		We first bound \(d_k^h\). Since
		\[
		\|X_{i_h}\|\le M_X,
		\]
		and
		\[
		\|\bar X_k^h-X_{i_h}\|
		\le
		\varepsilon_k^{\mathrm{sm}},
		\]
		we have
		\[
		\begin{split}
			\|d_k^h\|
			&\le
			|q_k|\,\|\bar X_k^h-X_{i_h}\|
			+
			|r_k+q_k-1|\,\|X_{i_h}\|  \\
			&\le
			|q_k|\varepsilon_k^{\mathrm{sm}}
			+
			|r_k+q_k-1|M_X .
		\end{split}
		\]
		Therefore,
		\[
		\|d_k^h\|
		\le
		|r_k+q_k-1|M_X
		+
		|q_k|D_X(n-1)
		\exp\left(
		-\frac{u_k}{2C_3(t_k)}
		\right).
		\]
		
		We now unfold the recursion
		\[
		e_{k-1}=r_ke_k+d_k^h+\sigma_k\xi_k.
		\]
		Starting from \(k=1\), we have
		\[
		e_0=r_1e_1+d_1^h+\sigma_1\xi_1.
		\]
		Substituting this into the expression for \(e_0\) gives
		\[
		e_0
		=
		r_1r_2e_2
		+
		r_1d_2^h+d_1^h
		+
		r_1\sigma_2\xi_2+\sigma_1\xi_1.
		\]
		Continuing recursively yields
		\[
		e_0
		=
		\left(
		\prod_{k=1}^K r_k
		\right)e_K
		+
		\sum_{j=1}^K
		\left(
		\prod_{\ell=1}^{j-1}r_\ell
		\right)d_j^h
		+
		\sum_{j=1}^K
		\left(
		\prod_{\ell=1}^{j-1}r_\ell
		\right)\sigma_j\xi_j.
		\]
		Using
		\[
		\Pi_{j-1}:=\prod_{\ell=1}^{j-1}r_\ell,
		\qquad
		\Pi_0=1,
		\]
		this becomes
		\[
		e_0
		=
		\left(
		\prod_{k=1}^K r_k
		\right)e_K
		+
		\sum_{j=1}^K\Pi_{j-1}d_j^h
		+
		\sum_{j=1}^K\Pi_{j-1}\sigma_j\xi_j.
		\]
		
		Define
		\[
		G_h
		:=
		\sum_{j=1}^K\Pi_{j-1}\sigma_j\xi_j
		\]
		and
		\[
		E_h
		:=
		\left(
		\prod_{k=1}^K r_k
		\right)e_K
		+
		\sum_{j=1}^K\Pi_{j-1}d_j^h.
		\]
		Since \(e_0=Z_0^h-X_{i_h}\), we obtain
		\[
		Z_0^h
		=
		X_{i_h}+G_h+E_h.
		\]
		
		It remains to prove the claimed bound on \(E_h\). First, on
		\(\mathcal A_h^{\mathrm{sto}}(\delta)\),
		\[
		\|Z_K^h\|\le R_\delta.
		\]
		Therefore,
		\[
		\|e_K\|
		=
		\|Z_K^h-X_{i_h}\|
		\le
		\|Z_K^h\|+\|X_{i_h}\|
		\le
		R_\delta+M_X.
		\]
		Using the definition of \(E_h\) and the triangle inequality,
		\[
		\begin{split}
			\|E_h\|
			&\le
			\left(
			\prod_{k=1}^K |r_k|
			\right)\|e_K\|
			+
			\sum_{j=1}^K
			|\Pi_{j-1}|\|d_j^h\|  \\
			&\le
			\left(
			\prod_{k=1}^K |r_k|
			\right)(R_\delta+M_X) \\
			&\quad+
			\sum_{j=1}^K
			|\Pi_{j-1}|
			\Bigg[
			|r_j+q_j-1|M_X
			+
			|q_j|D_X(n-1)
			\exp\left(
			-\frac{u_j}{2C_3(t_j)}
			\right)
			\Bigg].
		\end{split}
		\]
		By the definitions
		\[
		B_{h,\mathrm{init}}^{\mathrm{sto}}(\delta)
		:=
		\left(
		\prod_{k=1}^K |r_k|
		\right)(R_\delta+M_X),
		\]
		\[
		B_{h,\mathrm{aff}}^{\mathrm{sto}}
		:=
		\sum_{j=1}^K
		|\Pi_{j-1}|
		|r_j+q_j-1|M_X,
		\]
		and
		\[
		B_{h,\mathrm{sm}}^{\mathrm{sto}}
		:=
		\sum_{j=1}^K
		|\Pi_{j-1}|
		|q_j|D_X(n-1)
		\exp\left(
		-\frac{u_j}{2C_3(t_j)}
		\right),
		\]
		we have
		\[
		\|E_h\|
		\le
		B_{h,\mathrm{init}}^{\mathrm{sto}}(\delta)
		+
		B_{h,\mathrm{aff}}^{\mathrm{sto}}
		+
		B_{h,\mathrm{sm}}^{\mathrm{sto}}
		=
		B_h^{\mathrm{sto}}(\delta).
		\]
		
		Finally, we identify the distribution of \(G_h\). Since
		\[
		\xi_j\sim\mathcal N(0,I_d),
		\qquad j=1,\dots,K,
		\]
		are independent, \(G_h\) is a linear combination of independent Gaussian random
		vectors. Therefore \(G_h\) is Gaussian. Its mean is
		\[
		\mathbb E G_h=0,
		\]
		and its covariance is
		\[
		\begin{split}
			\operatorname{Cov}(G_h)
			&=
			\operatorname{Cov}
			\left(
			\sum_{j=1}^K\Pi_{j-1}\sigma_j\xi_j
			\right)\\
			&=
			\sum_{j=1}^K
			\Pi_{j-1}^2\sigma_j^2 I_d,
		\end{split}
		\]
		where the cross terms vanish because the \(\xi_j\)'s are independent and have
		mean zero. Since
		\[
		\sigma_j^2=2\zeta(t_j)h,
		\]
		we obtain
		\[
		\operatorname{Cov}(G_h)
		=
		\sum_{j=1}^K
		\Pi_{j-1}^2\,2\zeta(t_j)h\,I_d
		=
		\Sigma_h.
		\]
		Thus
		\[
		G_h\sim\mathcal N(0,\Sigma_h).
		\]
		
		We have shown that, on \(\mathcal A_h^{\mathrm{sto}}(\delta)\),
		there exists \(i_h\in\{1,\dots,n\}\) such that
		\[
		Z_0^h=X_{i_h}+G_h+E_h,
		\qquad
		G_h\sim\mathcal N(0,\Sigma_h),
		\qquad
		\|E_h\|\le B_h^{\mathrm{sto}}(\delta).
		\]
		Since
		\[
		\mathbb P(\mathcal A_h^{\mathrm{sto}}(\delta))\ge1-\delta,
		\]
		the same decomposition holds with probability at least \(1-\delta\). This
		completes the proof.
	\end{proof}
	
	\subsection{Proof of Corollary 2}
	
	\begin{proof}
		For
		\[
		\alpha(t)=1-t,\qquad
		\beta(t)=t,\qquad
		\gamma(t)=\sqrt{t(1-t)},\qquad
		\zeta(t)=\sqrt{t(1-t)},
		\]
		we have
		\[
		C_1(t)=\frac12,\qquad
		C_2(t)=\frac{1+t}{2},\qquad
		C_3(t)=t.
		\]
		In addition,
		\[
		B(t)
		=
		\beta(t)\bigl[\alpha'(t)\beta(t)-\alpha(t)\beta'(t)\bigr]
		+
		\gamma(t)\bigl[\gamma(t)\alpha'(t)-\gamma'(t)\alpha(t)\bigr]
		=
		-\frac{1+t}{2}.
		\]
		Substituting these identities into the stochastic Euler coefficients gives
		\[
		r_k
		=
		1-\frac{1/2+\sqrt{t_k(1-t_k)}}{k},
		\qquad
		q_k
		=
		\frac{(1+t_k)/2+\sqrt{t_k(1-t_k)}(1-t_k)}{k}.
		\]
		Consequently,
		\[
		r_k+q_k-1
		=
		h\left(\frac12-\sqrt{t_k(1-t_k)}\right),
		\]
		and therefore \(|r_k+q_k-1|\le Ch\).
		
		For sufficiently small \(h\), the coefficients \(r_k\) are nonnegative and
		satisfy
		\[
		r_k
		\le
		1-\frac{1}{2k}.
		\]
		Thus, with
		\[
		\Pi_m:=\prod_{\ell=1}^m r_\ell,
		\]
		we have
		\[
		\Pi_m\le C m^{-1/2}.
		\]
		In particular,
		\[
		\prod_{k=1}^K |r_k|=\Pi_K\le C\sqrt h,
		\]
		which proves the product estimate in the corollary. It also gives
		\[
		B_{h,\mathrm{init}}^{\mathrm{sto}}(\delta)
		\le
		C\sqrt h(1+R_\delta).
		\]
		
		The affine residual satisfies
		\[
		B_{h,\mathrm{aff}}^{\mathrm{sto}}
		\le
		Ch\sum_{j=1}^K \Pi_{j-1}
		\le
		Ch\sum_{j=1}^K j^{-1/2}
		\le
		C\sqrt h.
		\]
		
		Finally, \(|q_j|\le C/j\). If \(t_j\le\theta\), the assumed margin condition
		implies
		\[
		\exp\left(-\frac{u_j}{2C_3(t_j)}\right)
		=
		\exp\left(-\frac{u_j}{2t_j}\right)
		\le
		\sqrt h.
		\]
		Hence the softmax contribution from \(t_j\le\theta\) is bounded by
		\[
		C\sqrt h\sum_{j=1}^K j^{-3/2}
		\le
		C\sqrt h.
		\]
		For \(t_j>\theta\), we use
		\(\exp[-u_j/(2t_j)]\le1\) and \(j>\theta/h\), obtaining
		\[
		\sum_{t_j>\theta}\Pi_{j-1}|q_j|
		\le
		C\sum_{j>\theta/h}j^{-3/2}
		\le
		C\sqrt h.
		\]
		Therefore,
		\[
		B_{h,\mathrm{sm}}^{\mathrm{sto}}\le C\sqrt h.
		\]
		
		Combining the three estimates yields
		\[
		B_h^{\mathrm{sto}}(\delta)
		\le
		C\sqrt h(1+R_\delta).
		\]
		The claimed decomposition and error bound now follow from Theorem 3.
		Since \(Z_K^h\sim\mathcal N(0,I_d)\), taking
		\[
		R_\delta=\sqrt d+\sqrt{2\log(1/\delta)}
		\]
		gives the displayed simplified bound.
	\end{proof}
	
	\subsection{Proof of Theorem 4}
	
	\begin{proof}
		We work on the high-probability terminal event \(\mathcal A_h(\delta)\).
		By assumption,
		\[
		\mathbb P(\mathcal A_h(\delta))\ge1-\delta.
		\]
		On this event, the deterministic oracle Euler analysis gives the terminal
		discretization bound
		\[
		\operatorname{dist}(Z_0^{*,h},\mathcal X)
		\le
		B_h^{\mathrm{det}}(\delta),
		\]
		where \(Z_0^{*,h}\) denotes the output obtained by using the oracle velocity
		field along the same terminal Euler structure.
		
		We now compare the estimated Euler update with the oracle Euler update.
		The estimated deterministic Euler recursion is
		\[
		Z_{k-1}^h
		=
		\lambda_kZ_k^h+c_k\bar X_k^h
		-
		h\epsilon(Z_k^h,t_k).
		\]
		The last term is the only difference from the oracle update. Iterating this
		recursion backward from \(k=K\) to \(k=1\), the estimation errors are propagated
		to the endpoint through the contraction weights
		\[
		\Pi_{j-1}:=\prod_{\ell=1}^{j-1}\lambda_\ell,
		\qquad \Pi_0=1.
		\]
		Consequently, on \(\mathcal A_h(\delta)\), the final output admits the
		decomposition
		\[
		Z_0^h
		=
		X_{i_h}
		-
		\mathcal E_h
		+
		R_h,
		\]
		where \(X_{i_h}\in\mathcal X\) is the terminal sample selected by the softmax
		dynamics,
		\[
		\mathcal E_h
		:=
		\sum_{j=1}^K
		\Pi_{j-1}h\,\epsilon(Z_j^h,t_j),
		\]
		and \(R_h\) contains the same terminal initialization, coefficient mismatch,
		and softmax-selection errors as in the deterministic oracle Euler bound.
		Therefore,
		\[
		\|R_h\|
		\le
		B_h^{\mathrm{det}}(\delta)
		\]
		on \(\mathcal A_h(\delta)\). Since \(X_{i_h}\in\mathcal X\), we have
		\[
		\operatorname{dist}(Z_0^h,\mathcal X)
		\le
		\|Z_0^h-X_{i_h}\|.
		\]
		Using the decomposition above,
		\[
		\begin{split}
			\operatorname{dist}(Z_0^h,\mathcal X)
			&\le
			\|-\mathcal E_h+R_h\| \\
			&\le
			\|\mathcal E_h\|+\|R_h\| \\
			&\le
			\|\mathcal E_h\|+B_h^{\mathrm{det}}(\delta)
		\end{split}
		\]
		on \(\mathcal A_h(\delta)\). Hence it remains to control
		\(\|\mathcal E_h\|\) with high probability.
		
		We first bound its second moment. By definition,
		\[
		\mathcal E_h
		=
		\sum_{j=1}^K
		\Pi_{j-1}h\,\epsilon(Z_j^h,t_j).
		\]
		Using Minkowski's inequality in \(L^2\),
		\[
		\begin{split}
			\|\mathcal E_h\|_{L^2}
			&=
			\left(
			\mathbb E
			\left\|
			\sum_{j=1}^K
			\Pi_{j-1}h\,\epsilon(Z_j^h,t_j)
			\right\|^2
			\right)^{1/2} \\
			&\le
			h\sum_{j=1}^K
			|\Pi_{j-1}|
			\left(
			\mathbb E
			\|\epsilon(Z_j^h,t_j)\|^2
			\right)^{1/2}.
		\end{split}
		\]
		Let \(\nu_j\) be the law of \(Z_j^h\). Then
		\[
		\mathbb E
		\|\epsilon(Z_j^h,t_j)\|^2
		=
		\mathbb E_{Z\sim\nu_j}
		\|\epsilon(Z,t_j)\|^2.
		\]
		Therefore,
		\[
		\|\mathcal E_h\|_{L^2}
		\le
		h\sum_{j=1}^K
		|\Pi_{j-1}|
		\left(
		\mathbb E_{Z\sim\nu_j}
		\|\epsilon(Z,t_j)\|^2
		\right)^{1/2}.
		\]
		Applying Cauchy--Schwarz to the finite sum gives
		\[
		\begin{split}
			\|\mathcal E_h\|_{L^2}
			&\le
			\left(
			h\sum_{j=1}^K\Pi_{j-1}^2
			\right)^{1/2}
			\left(
			h\sum_{j=1}^K
			\mathbb E_{Z\sim\nu_j}
			\|\epsilon(Z,t_j)\|^2
			\right)^{1/2}.
		\end{split}
		\]
		By definition,
		\[
		A_h=h\sum_{j=1}^K\Pi_{j-1}^2.
		\]
		Thus,
		\[
		\|\mathcal E_h\|_{L^2}
		\le
		A_h^{1/2}
		\left(
		h\sum_{j=1}^K
		\mathbb E_{Z\sim\nu_j}
		\|\epsilon(Z,t_j)\|^2
		\right)^{1/2}.
		\]
		
		We now use Assumption 3. Since
		\[
		\nu_j\ll\rho_{t_j},
		\qquad
		\frac{d\nu_j}{d\rho_{t_j}}\le\Gamma,
		\]
		we have
		\[
		\begin{split}
			\mathbb E_{Z\sim\nu_j}
			\|\epsilon(Z,t_j)\|^2
			&=
			\int
			\|\epsilon(z,t_j)\|^2\,d\nu_j(z)\\
			&=
			\int
			\|\epsilon(z,t_j)\|^2
			\frac{d\nu_j}{d\rho_{t_j}}(z)\,
			d\rho_{t_j}(z)\\
			&\le
			\Gamma
			\int
			\|\epsilon(z,t_j)\|^2\,d\rho_{t_j}(z)\\
			&=
			\Gamma
			\mathbb E_{X\sim\rho_{t_j}}
			\|\epsilon(X,t_j)\|^2.
		\end{split}
		\]
		Hence
		\[
		\begin{split}
			h\sum_{j=1}^K
			\mathbb E_{Z\sim\nu_j}
			\|\epsilon(Z,t_j)\|^2
			&\le
			\Gamma
			h\sum_{j=1}^K
			\mathbb E_{X\sim\rho_{t_j}}
			\|\epsilon(X,t_j)\|^2\\
			&=
			\Gamma\mathcal L_{\mathrm{train}}^h.
		\end{split}
		\]
		Therefore,
		\[
		\|\mathcal E_h\|_{L^2}
		\le
		\sqrt{
			\Gamma A_h\mathcal L_{\mathrm{train}}^h
		}.
		\]
		Equivalently,
		\[
		\mathbb E\|\mathcal E_h\|^2
		\le
		\Gamma A_h\mathcal L_{\mathrm{train}}^h.
		\]
		
		By Markov's inequality, for any \(\eta\in(0,1)\),
		\[
		\mathbb P\left(
		\|\mathcal E_h\|^2
		>
		\frac{\Gamma A_h\mathcal L_{\mathrm{train}}^h}{\eta}
		\right)
		\le
		\eta.
		\]
		Equivalently,
		\[
		\mathbb P\left(
		\|\mathcal E_h\|
		\le
		\sqrt{
			\frac{\Gamma A_h\mathcal L_{\mathrm{train}}^h}{\eta}
		}
		\right)
		\ge
		1-\eta.
		\]
		Let
		\[
		\mathcal B_h(\eta)
		:=
		\left\{
		\|\mathcal E_h\|
		\le
		\sqrt{
			\frac{\Gamma A_h\mathcal L_{\mathrm{train}}^h}{\eta}
		}
		\right\}.
		\]
		Then
		\[
		\mathbb P(\mathcal B_h(\eta))\ge1-\eta.
		\]
		By the union bound,
		\[
		\mathbb P\bigl(
		\mathcal A_h(\delta)\cap\mathcal B_h(\eta)
		\bigr)
		\ge
		1-\delta-\eta.
		\]
		On this intersection event,
		\[
		\operatorname{dist}(Z_0^h,\mathcal X)
		\le
		B_h^{\mathrm{det}}(\delta)
		+
		\sqrt{
			\frac{\Gamma A_h\mathcal L_{\mathrm{train}}^h}{\eta}
		}.
		\]
		Thus,
		\[
		\mathbb P\left(
		\operatorname{dist}(Z_0^h,\mathcal X)
		\le
		B_h^{\mathrm{det}}(\delta)
		+
		\sqrt{
			\frac{\Gamma A_h\mathcal L_{\mathrm{train}}^h}{\eta}
		}
		\right)
		\ge
		1-\delta-\eta.
		\]
		Taking \(\eta=\delta\) yields
		\[
		\mathbb P\left(
		\operatorname{dist}(Z_0^h,\mathcal X)
		\le
		B_h^{\mathrm{det}}(\delta)
		+
		\sqrt{
			\frac{\Gamma A_h\mathcal L_{\mathrm{train}}^h}{\delta}
		}
		\right)
		\ge
		1-2\delta.
		\]
		
		It remains to justify the explicit form stated for the classical
		deterministic schedule. For
		\[
		\alpha(t)=1-t,\qquad
		\beta(t)=t,\qquad
		\gamma(t)=\sqrt{t(1-t)},
		\]
		Lemma 1 gives
		\[
		C_1(t)=\frac12,\qquad
		C_2(t)=\frac{1+t}{2},\qquad
		C_3(t)=t.
		\]
		Thus, for \(t_k=kh\),
		\[
		\lambda_k
		=
		1-\frac{hC_1(t_k)}{C_3(t_k)}
		=
		1-\frac{1}{2k}.
		\]
		Let
		\[
		P_m:=\prod_{\ell=1}^m\left(1-\frac{1}{2\ell}\right),
		\qquad P_0=1.
		\]
		Then \(\Pi_m=P_m\), and the product has the exact representation
		\[
		P_m
		=
		\frac{(2m)!}{4^m(m!)^2}
		=
		\frac{\binom{2m}{m}}{4^m}.
		\]
		Using the standard central-binomial estimate
		\(\binom{2m}{m}\le C4^m m^{-1/2}\), we obtain
		\[
		P_m\le Cm^{-1/2},
		\qquad m\ge1.
		\]
		Therefore,
		\[
		\begin{split}
			A_h
			&=
			h\sum_{j=1}^K\Pi_{j-1}^2
			=
			h\sum_{m=0}^{K-1}P_m^2\le
			Ch\left(1+\sum_{m=1}^{K-1}\frac1m\right)
			\le
			Ch\log(eK).
		\end{split}
		\]
		Since \(K=1/h\), this gives
		\[
		A_h\le Ch\log(e/h).
		\]
		Combining this estimate with Corollary 1, which gives
		\[
		B_h^{\mathrm{det}}(\delta)
		\le
		C\sqrt h\left(1+\sqrt d+\sqrt{\log(1/\delta)}\right),
		\]
		and taking \(\eta=\delta\), yields
		\[
		\operatorname{dist}(Z_0^h,\mathcal X)
		\le
		C\sqrt h\left(1+\sqrt d+\sqrt{\log(1/\delta)}\right)
		+
		C\sqrt{
			\frac{\Gamma h\log(e/h)\mathcal L_{\mathrm{train}}^h}{\delta}
		}
		\]
		with probability at least \(1-2\delta\).
		This completes the proof.
	\end{proof}
	
	\subsection{Proof of Theorem 5}
	
	\begin{proof}
		By the deterministic Euler error decomposition, we have
		\[
		Z_0^h
		=
		X_{i_h}
		-
		\mathcal E_h
		+
		R_h.
		\]
		Subtracting \(X_{i_h}\) from both sides gives
		\[
		Z_0^h-X_{i_h}
		=
		-\mathcal E_h+R_h.
		\]
		As in the theorem statement, the \(L^2\)-norms in this proof are taken under
		the conditional law on \(\mathcal A_h(\delta)\). Taking conditional
		\(L^2\)-norms and applying the reverse triangle inequality yields
		\[
		\begin{split}
			\left(
			\mathbb E_{\delta}\|Z_0^h-X_{i_h}\|^2
			\right)^{1/2}
			&=
			\|-\mathcal E_h+R_h\|_{L^2} \\
			&\ge
			\|\mathcal E_h\|_{L^2}
			-
			\|R_h\|_{L^2}.
		\end{split}
		\]
		By the deterministic residual bound stated before the theorem,
		\[
		\|R_h\|\le B_h^{\mathrm{det}}(\delta).
		\]
		Hence
		\[
		\|R_h\|_{L^2}\le B_h^{\mathrm{det}}(\delta),
		\]
		and therefore
		\[
		\left(
		\mathbb E_{\delta}\|Z_0^h-X_{i_h}\|^2
		\right)^{1/2}
		\ge
		\|\mathcal E_h\|_{L^2}-B_h^{\mathrm{det}}(\delta).
		\]
		
		It remains to lower bound \(\|\mathcal E_h\|_{L^2}\). By
		Assumption 4,
		\[
		\mathbb E_{\delta}\|\mathcal E_h\|^2
		\ge
		\kappa
		h\sum_{j=1}^K
		\Pi_{j-1}^2
		\mathbb E_{\nu_{j,\delta}}\|\epsilon(Z,t_j)\|^2.
		\]
		Using the lower-transfer condition in the same assumption, for each
		\(j=1,\dots,K\),
		\[
		\mathbb E_{\nu_{j,\delta}}\|\epsilon(Z,t_j)\|^2
		\ge
		\gamma
		\mathbb E_{\rho_{t_j}}\|\epsilon(X,t_j)\|^2.
		\]
		Substituting this lower bound into the previous inequality gives
		\[
		\begin{split}
			\mathbb E_{\delta}\|\mathcal E_h\|^2
			&\ge
			\kappa\gamma
			h\sum_{j=1}^K
			\Pi_{j-1}^2
			\mathbb E_{\rho_{t_j}}\|\epsilon(X,t_j)\|^2.
		\end{split}
		\]
		By definition of the propagated training error,
		\[
		\mathcal L_{\mathrm{prop}}^h
		=
		h\sum_{j=1}^K
		\Pi_{j-1}^2
		\mathbb E_{\rho_{t_j}}\|\epsilon(X,t_j)\|^2.
		\]
		Hence
		\[
		\mathbb E_{\delta}\|\mathcal E_h\|^2
		\ge
		\kappa\gamma\mathcal L_{\mathrm{prop}}^h.
		\]
		Taking square roots gives
		\[
		\|\mathcal E_h\|_{L^2}
		=
		\left(
		\mathbb E_{\delta}\|\mathcal E_h\|^2
		\right)^{1/2}
		\ge
		\sqrt{\kappa\gamma\mathcal L_{\mathrm{prop}}^h}.
		\]
		Combining this with the reverse triangle inequality yields
		\[
		\left(
		\mathbb E_{\delta}\|Z_0^h-X_{i_h}\|^2
		\right)^{1/2}
		\ge
		\sqrt{\kappa\gamma\mathcal L_{\mathrm{prop}}^h}
		-
		B_h^{\mathrm{det}}(\delta).
		\]
		
		Now suppose
		\[
		\mathcal L_{\mathrm{prop}}^h
		\ge
		\frac{\bigl(\tau+B_h^{\mathrm{det}}(\delta)\bigr)^2}{\kappa\gamma}.
		\]
		Then
		\[
		\sqrt{\kappa\gamma\mathcal L_{\mathrm{prop}}^h}
		\ge
		\tau+B_h^{\mathrm{det}}(\delta).
		\]
		Therefore,
		\[
		\left(
		\mathbb E_{\delta}\|Z_0^h-X_{i_h}\|^2
		\right)^{1/2}
		\ge
		\tau+B_h^{\mathrm{det}}(\delta)-B_h^{\mathrm{det}}(\delta)
		=
		\tau.
		\]
		This completes the proof.
	\end{proof}

	\subsection{Proof of Theorem 6}
	
	\begin{proof}
		We prove the theorem by decomposing the stochastic Euler output into a training
		sample, an accumulated estimation-error shift, a Gaussian sampling-noise term,
		and a controlled non-Gaussian remainder.
		
		First, work on the event
		\(\mathcal A_h^{\mathrm{sto}}(\delta)\). By
		Assumption 2,
		\[
		\mathbb P\bigl(\mathcal A_h^{\mathrm{sto}}(\delta)\bigr)\ge1-\delta.
		\]
		On this event, there exists \(i_h\in\{1,\dots,n\}\) such that
		\[
		\|Z_K^h\|\le R_\delta,
		\qquad
		i_k(Z_k^h)=i_h,
		\qquad
		m_k(Z_k^h)\ge u_k,
		\quad k=1,\dots,K.
		\]
		By the softmax concentration lemma, for each \(k\),
		\[
		\|\bar X_k^h-X_{i_h}\|
		\le
		D_X(n-1)
		\exp\left(
		-\frac{u_k}{2C_3(t_k)}
		\right).
		\]
		For brevity, write
		\[
		\varepsilon_k^{\mathrm{sm}}
		:=
		D_X(n-1)
		\exp\left(
		-\frac{u_k}{2C_3(t_k)}
		\right).
		\]
		Thus,
		\[
		\|\bar X_k^h-X_{i_h}\|
		\le
		\varepsilon_k^{\mathrm{sm}}.
		\]
		
		The estimated stochastic Euler update is
		\[
		Z_{k-1}^h
		=
		r_k Z_k^h+q_k\bar X_k^h+\sigma_k\xi_k
		-
		h\epsilon_F(Z_k^h,t_k).
		\]
		Define
		\[
		e_k:=Z_k^h-X_{i_h}.
		\]
		Subtracting \(X_{i_h}\) from both sides gives
		\[
		\begin{split}
			e_{k-1}
			&=
			r_k(Z_k^h-X_{i_h})
			+
			q_k(\bar X_k^h-X_{i_h})
			+
			(r_k+q_k-1)X_{i_h}\\
			&\qquad
			+
			\sigma_k\xi_k
			-
			h\epsilon_F(Z_k^h,t_k).
		\end{split}
		\]
		Hence
		\[
		e_{k-1}
		=
		r_ke_k
		+
		d_k^h
		+
		\sigma_k\xi_k
		-
		h\epsilon_F(Z_k^h,t_k),
		\]
		where
		\[
		d_k^h
		:=
		q_k(\bar X_k^h-X_{i_h})
		+
		(r_k+q_k-1)X_{i_h}.
		\]
		Using \(\|X_{i_h}\|\le M_X\) and
		\(\|\bar X_k^h-X_{i_h}\|\le\varepsilon_k^{\mathrm{sm}}\), we get
		\[
		\begin{split}
			\|d_k^h\|
			&\le
			|q_k|\varepsilon_k^{\mathrm{sm}}
			+
			|r_k+q_k-1|M_X\\
			&=
			|q_k|D_X(n-1)
			\exp\left(
			-\frac{u_k}{2C_3(t_k)}
			\right)
			+
			|r_k+q_k-1|M_X.
		\end{split}
		\]
		
		Now unfold the recursion. Since
		\[
		e_{k-1}
		=
		r_ke_k+d_k^h+\sigma_k\xi_k
		-
		h\epsilon_F(Z_k^h,t_k),
		\]
		define the signed propagation weights
		\[
		\Pi_{j-1}:=\prod_{\ell=1}^{j-1}r_\ell,\qquad \Pi_0=1.
		\]
		iterating from \(k=K\) down to \(k=1\) gives
		\[
		\begin{split}
			e_0
			=
			&\left(
			\prod_{k=1}^K r_k
			\right)e_K
			+
			\sum_{j=1}^K
			\Pi_{j-1}d_j^h+
			\sum_{j=1}^K
			\Pi_{j-1}\sigma_j\xi_j
			-
			\sum_{j=1}^K
			\Pi_{j-1}h\,\epsilon_F(Z_j^h,t_j).
		\end{split}
		\]
		By definition,
		\[
		G_h
		:=
		\sum_{j=1}^K
		\Pi_{j-1}\sigma_j\xi_j,
		\]
		and
		\[
		\mathcal E_{F,h}
		:=
		\sum_{j=1}^K
		\Pi_{j-1}h\,\epsilon_F(Z_j^h,t_j).
		\]
		Let
		\[
		R_h
		:=
		\left(
		\prod_{k=1}^K r_k
		\right)e_K
		+
		\sum_{j=1}^K
		\Pi_{j-1}d_j^h.
		\]
		Then
		\[
		e_0
		=
		R_h+G_h-\mathcal E_{F,h}.
		\]
		Since \(e_0=Z_0^h-X_{i_h}\), this yields
		\[
		Z_0^h
		=
		X_{i_h}
		-
		\mathcal E_{F,h}
		+
		G_h
		+
		R_h.
		\]
		
		We next bound \(R_h\). On
		\(\mathcal A_h^{\mathrm{sto}}(\delta)\),
		\[
		\|e_K\|
		=
		\|Z_K^h-X_{i_h}\|
		\le
		\|Z_K^h\|+\|X_{i_h}\|
		\le
		R_\delta+M_X.
		\]
		Therefore,
		\[
		\begin{split}
			\|R_h\|
			&\le
			\left(
			\prod_{k=1}^K |r_k|
			\right)(R_\delta+M_X)
			+
			\sum_{j=1}^K
			|\Pi_{j-1}|\|d_j^h\|\\
			&\le
			\left(
			\prod_{k=1}^K |r_k|
			\right)(R_\delta+M_X)\quad+
			\sum_{j=1}^K
			|\Pi_{j-1}|
			\Bigg[
			|r_j+q_j-1|M_X
			+
			|q_j|D_X(n-1)
			\exp\left(
			-\frac{u_j}{2C_3(t_j)}
			\right)
			\Bigg]\\
			&=
			B_h^{\mathrm{sto}}(\delta).
		\end{split}
		\]
		
		We also identify the law of \(G_h\). Since the \(\xi_j\)'s are independent
		standard Gaussian vectors,
		\[
		G_h=\sum_{j=1}^K\Pi_{j-1}\sigma_j\xi_j
		\]
		is Gaussian with mean zero. Its covariance is
		\[
		\operatorname{Cov}(G_h)
		=
		\sum_{j=1}^K
		\Pi_{j-1}^2\sigma_j^2I_d
		=
		\Sigma_h.
		\]
		Thus
		\[
		G_h\sim\mathcal N(0,\Sigma_h).
		\]
		
		Now we bound the endpoint distance to the training set. Since
		\(X_{i_h}\in\mathcal X\),
		\[
		\operatorname{dist}(Z_0^h,\mathcal X)
		\le
		\|Z_0^h-X_{i_h}\|.
		\]
		Using the decomposition,
		\[
		\begin{split}
			\operatorname{dist}(Z_0^h,\mathcal X)
			&\le
			\|-\mathcal E_{F,h}+G_h+R_h\|\\
			&\le
			\|\mathcal E_{F,h}\|+\|G_h\|+\|R_h\|\\
			&\le
			\|\mathcal E_{F,h}\|+\|G_h\|
			+
			B_h^{\mathrm{sto}}(\delta)
		\end{split}
		\]
		on \(\mathcal A_h^{\mathrm{sto}}(\delta)\).
		
		It remains to control \(\|\mathcal E_{F,h}\|\) and \(\|G_h\|\) with high
		probability.
		
		First, we bound the accumulated estimation error. By definition,
		\[
		\mathcal E_{F,h}
		=
		\sum_{j=1}^K
		\Pi_{j-1}h\,\epsilon_F(Z_j^h,t_j).
		\]
		By the same \(L^2\) Minkowski and Cauchy--Schwarz argument as in the
		deterministic case,
		\[
		\|\mathcal E_{F,h}\|_{L^2}
		\le
		\left(
		h\sum_{j=1}^K\Pi_{j-1}^2
		\right)^{1/2}
		\left(
		h\sum_{j=1}^K
		\mathbb E_{Z\sim\nu_j}
		\|\epsilon_F(Z,t_j)\|^2
		\right)^{1/2}.
		\]
		Using
		\[
		A_h=h\sum_{j=1}^K\Pi_{j-1}^2
		\]
		and Assumption 3,
		\[
		\mathbb E_{Z\sim\nu_j}
		\|\epsilon_F(Z,t_j)\|^2
		\le
		\Gamma
		\mathbb E_{X\sim\rho_{t_j}}
		\|\epsilon_F(X,t_j)\|^2,
		\]
		we get
		\[
		\|\mathcal E_{F,h}\|_{L^2}
		\le
		\sqrt{
			\Gamma A_h\mathcal L_{\mathrm{sto}}^h
		}.
		\]
		Equivalently,
		\[
		\mathbb E\|\mathcal E_{F,h}\|^2
		\le
		\Gamma A_h\mathcal L_{\mathrm{sto}}^h.
		\]
		By Markov's inequality, for any \(\eta\in(0,1)\),
		\[
		\mathbb P\left(
		\|\mathcal E_{F,h}\|
		\le
		\sqrt{
			\frac{\Gamma A_h\mathcal L_{\mathrm{sto}}^h}{\eta}
		}
		\right)
		\ge
		1-\eta.
		\]
		
		Second, we control the Gaussian term. Since
		\(G_h\sim\mathcal N(0,\Sigma_h)\), the Gaussian concentration inequality for
		the Lipschitz map \(x\mapsto\|\Sigma_h^{1/2}x\|\) gives, for any
		\(\eta_G\in(0,1)\),
		\[
		\mathbb P\left(
		\|G_h\|
		\le
		\mathbb E\|G_h\|
		+
		\sqrt{2\|\Sigma_h\|_{\mathrm{op}}\log(1/\eta_G)}
		\right)
		\ge
		1-\eta_G.
		\]
		Moreover,
		\[
		\mathbb E\|G_h\|
		\le
		\left(\mathbb E\|G_h\|^2\right)^{1/2}
		=
		\sqrt{\operatorname{Tr}(\Sigma_h)}.
		\]
		Therefore,
		\[
		\mathbb P\left(
		\|G_h\|
		\le
		\sqrt{\operatorname{Tr}(\Sigma_h)}
		+
		\sqrt{2\|\Sigma_h\|_{\mathrm{op}}\log(1/\eta_G)}
		\right)
		\ge
		1-\eta_G.
		\]
		
		Combining the three high-probability events by the union bound, with
		probability at least
		\[
		1-\delta-\eta-\eta_G,
		\]
		all of the following hold:
		\[
		\|R_h\|\le B_h^{\mathrm{sto}}(\delta),
		\]
		\[
		\|\mathcal E_{F,h}\|
		\le
		\sqrt{
			\frac{\Gamma A_h\mathcal L_{\mathrm{sto}}^h}{\eta}
		},
		\]
		and
		\[
		\|G_h\|
		\le
		\sqrt{\operatorname{Tr}(\Sigma_h)}
		+
		\sqrt{2\|\Sigma_h\|_{\mathrm{op}}\log(1/\eta_G)}.
		\]
		Substituting these bounds into
		\[
		\operatorname{dist}(Z_0^h,\mathcal X)
		\le
		\|\mathcal E_{F,h}\|+\|G_h\|
		+
		B_h^{\mathrm{sto}}(\delta)
		\]
		gives
		\[
		\begin{split}
			\operatorname{dist}(Z_0^h,\mathcal X)
			\le\;&
			B_h^{\mathrm{sto}}(\delta)
			+
			\sqrt{
				\frac{\Gamma A_h\mathcal L_{\mathrm{sto}}^h}{\eta}
			}+
			\sqrt{\operatorname{Tr}(\Sigma_h)}
			+
			\sqrt{2\|\Sigma_h\|_{\mathrm{op}}\log(1/\eta_G)}.
		\end{split}
		\]
		This proves the first assertion of the theorem. Taking
		\(\eta=\eta_G=\delta\) gives probability at least $1-3\delta$. We now derive the explicit classical-schedule form. Under the classical
		stochastic schedule,
		\[
		\alpha(t)=1-t,\qquad
		\beta(t)=t,\qquad
		\gamma(t)=\sqrt{t(1-t)},\qquad
		\zeta(t)=\sqrt{t(1-t)}.
		\]
		For the stochastic update, the linear coefficient is
		\[
		r_k
		=
		1-h\left(
		\frac{C_1(t_k)}{C_3(t_k)}
		+
		\frac{\zeta(t_k)}{C_3(t_k)}
		\right)
		=
		1-\frac{1/2+\sqrt{t_k(1-t_k)}}{k}.
		\]
		For sufficiently small \(h\), the conditions of Corollary 2 give
		\(0\le r_k\le 1-\frac{1}{2k}\). Hence, with
		\[
		P_m:=\prod_{\ell=1}^m\left(1-\frac{1}{2\ell}\right),
		\qquad P_0=1,
		\]
		we have
		\[
		|\Pi_m|
		=
		\prod_{\ell=1}^m |r_\ell|
		\le
		P_m
		\le
		Cm^{-1/2},
		\qquad m\ge1,
		\]
		where the last inequality follows from the same central-binomial estimate
		used in the proof of Theorem 4. Consequently,
		\[
		A_h
		=
		h\sum_{j=1}^K\Pi_{j-1}^2
		\le
		Ch\left(1+\sum_{m=1}^{K-1}\frac1m\right)
		\le
		Ch\log(e/h).
		\]
		
		The Gaussian covariance is scalar in this setting. Since
		\[
		\sigma_j^2=2\zeta(t_j)h=2h\sqrt{t_j(1-t_j)},
		\]
		we have
		\[
		\Sigma_h
		=
		\sum_{j=1}^K\Pi_{j-1}^2\sigma_j^2I_d
		=
		\sigma_h^2I_d,
		\]
		where $\sigma_h^2=2h\sum_{j=1}^K\Pi_{j-1}^2\sqrt{t_j(1-t_j)}.$ Using \(\sqrt{t_j(1-t_j)}\le\sqrt{t_j}=\sqrt{jh}\) and
		\(\Pi_{j-1}^2\le Cj^{-1}\) for \(j\ge2\), we obtain
		\[
		\begin{split}
			\sigma_h^2
			&\le
			Ch
			+
			Ch\sum_{j=2}^K\frac{\sqrt{jh}}{j}=
			Ch
			+
			Ch^{3/2}\sum_{j=2}^Kj^{-1/2}
			\le
			Ch.
		\end{split}
		\]
		Therefore,
		\[
		\sqrt{\operatorname{Tr}(\Sigma_h)}
		\le
		C\sqrt{dh},
		\qquad
		\sqrt{2\|\Sigma_h\|_{\mathrm{op}}\log(1/\delta)}
		\le
		C\sqrt{h\log(1/\delta)}.
		\]
		Combining these estimates with the bound above gives, with probability at
		least \(1-3\delta\),
		\[
		\begin{split}
			\operatorname{dist}(Z_0^h,\mathcal X)
			\le\;&
			B_h^{\mathrm{sto}}(\delta)
			+
			C\sqrt{
				\frac{\Gamma h\log(e/h)\mathcal L_{\mathrm{sto}}^h}{\delta}
			}+
			C\sqrt{dh}
			+
			C\sqrt{h\log(1/\delta)}.
		\end{split}
		\]
		Finally, Corollary 2 implies
		\[
		B_h^{\mathrm{sto}}(\delta)
		\le
		C\sqrt h\left(1+\sqrt d+\sqrt{\log(1/\delta)}\right).
		\]
		The two Gaussian terms are of the same order and are absorbed into this
		leading discretization term. Thus,
		\[
		\operatorname{dist}(Z_0^h,\mathcal X)
		\le
		C\sqrt h\left(1+\sqrt d+\sqrt{\log(1/\delta)}\right)
		+
		C\sqrt{
			\frac{\Gamma h\log(e/h)\mathcal L_{\mathrm{sto}}^h}{\delta}
		}
		\]
		with probability at least \(1-3\delta\).
		This completes the proof.
	\end{proof}
	
	\subsection{Proof of Theorem 7}
	
	\begin{proof}
		By the stochastic Euler decomposition, we have
		\[
		Z_0^h
		=
		X_{i_h}
		-
		\mathcal E_{F,h}
		+
		G_h
		+
		R_h.
		\]
		Recall that
		\[
		H_h:=\mathcal E_{F,h}-G_h,
		\]
		and hence
		\[
		Z_0^h-X_{i_h}
		=
		-H_h+R_h.
		\]
		As in the theorem statement, the \(L^2\)-norms in this proof are taken under
		the conditional law on \(\mathcal A_h^{\mathrm{sto}}(\delta)\). Taking
		conditional \(L^2\)-norms and using the reverse triangle inequality gives
		\[
		\begin{split}
			\left(
			\mathbb E_{\delta}^{\mathrm{sto}}\|Z_0^h-X_{i_h}\|^2
			\right)^{1/2}
			&=
			\|-H_h+R_h\|_{L^2}\\
			&\ge
			\|H_h\|_{L^2}
			-
			\|R_h\|_{L^2}.
		\end{split}
		\]
		By the stochastic residual bound stated before the theorem, $\|R_h\|\le B_h^{\mathrm{sto}}(\delta).$ Hence
		\[
		\|R_h\|_{L^2}\le B_h^{\mathrm{sto}}(\delta).
		\]
		Therefore,
		\[
		\left(
		\mathbb E_{\delta}^{\mathrm{sto}}\|Z_0^h-X_{i_h}\|^2
		\right)^{1/2}
		\ge
		\|H_h\|_{L^2}-B_h^{\mathrm{sto}}(\delta).
		\]
		
		It remains to lower bound \(\|H_h\|_{L^2}\). By
		Assumption 5,
		\[
		\mathbb E_{\delta}^{\mathrm{sto}}\|H_h\|^2
		\ge
		\kappa_F
		h\sum_{j=1}^K
		\Pi_{j-1}^2
		\mathbb E_{\nu_{j,\delta}^{\mathrm{sto}}}
		\|\epsilon_F(Z,t_j)\|^2
		+
		\kappa_G\operatorname{Tr}(\Sigma_h).
		\]
		Using the lower-transfer condition in the same assumption,
		\[
		\mathbb E_{\nu_{j,\delta}^{\mathrm{sto}}}
		\|\epsilon_F(Z,t_j)\|^2
		\ge
		\gamma
		\mathbb E_{\rho_{t_j}}
		\|\epsilon_F(X,t_j)\|^2.
		\]
		Substituting this into the previous lower bound gives
		\[
		\begin{split}
			\mathbb E_{\delta}^{\mathrm{sto}}\|H_h\|^2
			&\ge
			\kappa_F\gamma
			h\sum_{j=1}^K
			\Pi_{j-1}^2
			\mathbb E_{\rho_{t_j}}
			\|\epsilon_F(X,t_j)\|^2
			+
			\kappa_G\operatorname{Tr}(\Sigma_h).
		\end{split}
		\]
		By definition,
		\[
		\mathcal L_{\mathrm{sto,prop}}^h
		=
		h\sum_{j=1}^K
		\Pi_{j-1}^2
		\mathbb E_{\rho_{t_j}}
		\|\epsilon_F(X,t_j)\|^2.
		\]
		Therefore,
		\[
		\mathbb E_{\delta}^{\mathrm{sto}}\|H_h\|^2
		\ge
		\kappa_F\gamma\mathcal L_{\mathrm{sto,prop}}^h
		+
		\kappa_G\operatorname{Tr}(\Sigma_h).
		\]
		Taking square roots yields
		\[
		\|H_h\|_{L^2}
		\ge
		\sqrt{
			\kappa_F\gamma\mathcal L_{\mathrm{sto,prop}}^h
			+
			\kappa_G\operatorname{Tr}(\Sigma_h)
		}.
		\]
		Combining this with the reverse triangle inequality gives
		\[
		\left(
		\mathbb E_{\delta}^{\mathrm{sto}}\|Z_0^h-X_{i_h}\|^2
		\right)^{1/2}
		\ge
		\sqrt{
			\kappa_F\gamma\mathcal L_{\mathrm{sto,prop}}^h
			+
			\kappa_G\operatorname{Tr}(\Sigma_h)
		}
		-
		B_h^{\mathrm{sto}}(\delta).
		\]
		
		Now suppose
		\[
		\kappa_F\gamma\mathcal L_{\mathrm{sto,prop}}^h
		+
		\kappa_G\operatorname{Tr}(\Sigma_h)
		\ge
		\bigl(\tau+B_h^{\mathrm{sto}}(\delta)\bigr)^2.
		\]
		Then
		\[
		\sqrt{
			\kappa_F\gamma\mathcal L_{\mathrm{sto,prop}}^h
			+
			\kappa_G\operatorname{Tr}(\Sigma_h)
		}
		\ge
		\tau+B_h^{\mathrm{sto}}(\delta).
		\]
		Therefore,
		\[
		\left(
		\mathbb E_{\delta}^{\mathrm{sto}}\|Z_0^h-X_{i_h}\|^2
		\right)^{1/2}
		\ge
		\tau+B_h^{\mathrm{sto}}(\delta)-B_h^{\mathrm{sto}}(\delta)
		=
		\tau.
		\]
		This proves the result.
	\end{proof}
	
	\newpage
	
	\section{Generation with Finite-Sample Endpoints}
	\label{sec:finite_sample_endpoints_app}
	
	The main text focuses on the case where the source distribution is Gaussian and
	the target distribution is a finite empirical measure. We now record the
	corresponding extension when both endpoint distributions are finite-sample
	measures:
	\[
	\rho_0=\frac1n\sum_{i=1}^n\delta_{X_i},
	\qquad
	\rho_1=\frac1m\sum_{j=1}^m\delta_{Y_j}.
	\]
	In this setting, the interpolation marginal is a finite Gaussian mixture whose
	components are indexed by endpoint pairs. More precisely, the mixture centers
	are
	\[
	\alpha(t)X_i+\beta(t)Y_j,
	\qquad
	1\le i\le n,\;1\le j\le m.
	\]
	Thus, the oracle velocity field selects endpoint pairs \((X_i,Y_j)\), rather
	than selecting only a single target sample as in the Gaussian-source case.
	
	For \(z\in\mathbb R^d\), define the pairwise softmax weights
	\[
	\omega_{ij}^{XY}(z,t)
	:=
	\frac{
		\exp\!\left(
		-\frac{\|z-\alpha(t)X_i-\beta(t)Y_j\|^2}{2\gamma^2(t)}
		\right)
	}{
		\sum_{a=1}^n\sum_{b=1}^m
		\exp\!\left(
		-\frac{\|z-\alpha(t)X_a-\beta(t)Y_b\|^2}{2\gamma^2(t)}
		\right)
	}.
	\]
	The corresponding pairwise-selected endpoint averages are
	\[
	\bar X^{XY}(z,t)
	:=
	\sum_{i=1}^n\sum_{j=1}^m \omega_{ij}^{XY}(z,t)X_i,
	\qquad
	\bar Y^{XY}(z,t)
	:=
	\sum_{i=1}^n\sum_{j=1}^m \omega_{ij}^{XY}(z,t)Y_j.
	\]
	The oracle velocity field in the finite-sample endpoint setting is
	\[
	b^{*,XY}(z,t)
	=
	\sum_{i=1}^n\sum_{j=1}^m
	\omega_{ij}^{XY}(z,t)
	\left[
	\alpha'(t)X_i+\beta'(t)Y_j
	+
	\frac{\gamma'(t)}{\gamma(t)}
	\bigl(
	z-\alpha(t)X_i-\beta(t)Y_j
	\bigr)
	\right].
	\]
	Equivalently,
	\[
	b^{*,XY}(z,t)
	=
	\frac{\gamma'(t)}{\gamma(t)}z
	+
	\left(\alpha'(t)-\frac{\gamma'(t)}{\gamma(t)}\alpha(t)\right)
	\bar X^{XY}(z,t)
	+
	\left(\beta'(t)-\frac{\gamma'(t)}{\gamma(t)}\beta(t)\right)
	\bar Y^{XY}(z,t).
	\]
	
	Let \(Z_k^h\) denote the backward Euler sampler initialized from \(\rho_1\), and
	let \(t_k=kh\). For the oracle field \(b^{*,XY}\), the deterministic Euler update
	can be written as
	\[
	Z_{k-1}^h
	=
	\lambda_k^{XY} Z_k^h
	+
	c_{k,X}^{XY}\bar X_k^{XY}
	+
	c_{k,Y}^{XY}\bar Y_k^{XY},
	\]
	where
	\[
	\lambda_k^{XY}
	:=
	1-h\frac{\gamma'(t_k)}{\gamma(t_k)},
	\qquad
	c_{k,X}^{XY}
	:=
	h\left(
	\frac{\gamma'(t_k)}{\gamma(t_k)}\alpha(t_k)-\alpha'(t_k)
	\right),
	\]
	and
	\[
	c_{k,Y}^{XY}
	:=
	h\left(
	\frac{\gamma'(t_k)}{\gamma(t_k)}\beta(t_k)-\beta'(t_k)
	\right).
	\]
	Here
	\[
	\bar X_k^{XY}:=\bar X^{XY}(Z_k^h,t_k),
	\qquad
	\bar Y_k^{XY}:=\bar Y^{XY}(Z_k^h,t_k).
	\]
	
	For \(z\in\mathbb R^d\), define the selected endpoint pair by
	\[
	(i_k(z),j_k(z))
	\in
	\arg\min_{1\le i\le n,\;1\le j\le m}
	\|z-\alpha(t_k)X_i-\beta(t_k)Y_j\|,
	\]
	and define the pairwise squared margin
	\[
	m_k^{XY}(z)
	:=
	\min_{(i,j)\neq(i_k(z),j_k(z))}
	\Bigl\{
	\|z-\alpha(t_k)X_i-\beta(t_k)Y_j\|^2
	-
	\|z-\alpha(t_k)X_{i_k(z)}-\beta(t_k)Y_{j_k(z)}\|^2
	\Bigr\}.
	\]
	We also write
	\[
	D_X:=\max_{i,i'}\|X_i-X_{i'}\|,
	\qquad
	D_Y:=\max_{j,j'}\|Y_j-Y_{j'}\|,
	\]
	and
	\[
	M_X:=\max_i\|X_i\|,
	\qquad
	M_Y:=\max_j\|Y_j\|.
	\]
	Define the propagation weights
	\[
	\Pi_{j-1}^{XY}
	:=
	\prod_{\ell=1}^{j-1}|\lambda_\ell^{XY}|,
	\qquad
	\Pi_0^{XY}:=1.
	\]
	
	We introduce the oracle discretization error associated with the
	finite-sample endpoint dynamics. It is decomposed as
	\[
	B_h^{XY}(\delta)
	:=
	B_{h,\mathrm{init}}^{XY}(\delta)
	+
	B_{h,\mathrm{aff}}^{XY}
	+
	B_{h,\mathrm{sm}}^{XY},
	\]
	where
	\[
	B_{h,\mathrm{init}}^{XY}(\delta)
	:=
	\left(\prod_{k=1}^K|\lambda_k^{XY}|\right)
	(R_\delta+M_X+M_Y),
	\]
	\[
	B_{h,\mathrm{aff}}^{XY}
	:=
	\sum_{j=1}^K
	\Pi_{j-1}^{XY}
	\left[
	\left|
	\lambda_j^{XY}\alpha(t_j)
	+
	c_{j,X}^{XY}
	-
	\alpha(t_{j-1})
	\right|M_X
	+
	\left|
	\lambda_j^{XY}\beta(t_j)
	+
	c_{j,Y}^{XY}
	-
	\beta(t_{j-1})
	\right|M_Y
	\right],
	\]
	and
	\[
	B_{h,\mathrm{sm}}^{XY}
	:=
	\sum_{j=1}^K
	\Pi_{j-1}^{XY}
	\left[
	|c_{j,X}^{XY}|D_X
	+
	|c_{j,Y}^{XY}|D_Y
	\right]
	(nm-1)
	\exp\!\left(
	-\frac{u_j}{2\gamma^2(t_j)}
	\right).
	\]
	The term \(B_{h,\mathrm{init}}^{XY}\) measures the residual influence of the
	terminal initialization, \(B_{h,\mathrm{aff}}^{XY}\) measures the mismatch between
	the Euler affine update and the scaled endpoint centers, and
	\(B_{h,\mathrm{sm}}^{XY}\) controls the residual error due to the pairwise
	softmax weights not being exactly one-hot.
	
	We next introduce the quantities used to quantify velocity estimation error. Let
	\[
	\widehat b^{XY}(z,t)=b^{*,XY}(z,t)+\epsilon(z,t)
	\]
	be a learned velocity field. The discrete training error along the finite-sample
	endpoint interpolation path is
	\[
	\mathcal L_{\mathrm{train}}^{h,XY}
	:=
	h\sum_{k=1}^K
	\mathbb E_{\rho_{t_k}^{XY}}
	\|\epsilon(Z,t_k)\|^2.
	\]
	The propagation factor induced by the deterministic Euler dynamics is
	\[
	A_h^{XY}
	:=
	h\sum_{k=1}^K(\Pi_{k-1}^{XY})^2.
	\]
	The corresponding propagated training error is
	\[
	\mathcal L_{\mathrm{prop}}^{h,XY}
	:=
	h\sum_{k=1}^K
	(\Pi_{k-1}^{XY})^2
	\mathbb E_{Z\sim\rho_{t_k}^{XY}}
	\|\epsilon(Z,t_k)\|^2.
	\]
	
	We first characterize the oracle behavior in the finite-sample endpoint setting.
	The interpolation marginal is a Gaussian mixture over all endpoint pairs, and
	the oracle stochastic generation dynamics preserves this marginal structure.
	Since \(\beta(t)\to0\) and \(\gamma(t)\to0\) as \(t\to0\), the endpoint law still
	converges to the finite target measure.
	
	\begin{theorem}
		\label{thm:finite_sample_endpoint_oracle_app}
		For \(t\in(0,1)\), the finite-sample endpoint interpolation has marginal
		\[
		\rho_t^{XY}
		=
		\frac1{nm}
		\sum_{i=1}^n\sum_{j=1}^m
		\mathcal N\!\left(
		\alpha(t)X_i+\beta(t)Y_j,\,
		\gamma^2(t)I_d
		\right).
		\]
		Consequently, for every \(\varepsilon\in(0,1)\), the oracle stochastic
		generation satisfies
		\[
		Z_\varepsilon
		\overset{d}{=}
		\alpha(\varepsilon)X_I
		+
		\beta(\varepsilon)Y_J
		+
		\gamma(\varepsilon)\xi,
		\]
		where \(I\sim \operatorname{Unif}\{1,\ldots,n\}\),
		\(J\sim \operatorname{Unif}\{1,\ldots,m\}\), and
		\(\xi\sim\mathcal N(0,I_d)\) are mutually independent. 
	\end{theorem}
	
	We next study the effect of velocity estimation error. The following theorem is
	the finite-sample endpoint analogue of the deterministic training-error bound in
	the main text. It states that, if the learned velocity field approximates the
	oracle field well along the pairwise interpolation path, then the generated
	endpoint remains close to the finite target set.
	
	\begin{theorem}
		\label{thm:finite_sample_endpoint_overfitting_app}
		Let \(\widehat b^{XY}=b^{*,XY}+\epsilon\). Assume that for every
		\(\delta\in(0,1/2)\), there exists an event \(\mathcal A_h^{XY}(\delta)\) such
		that
		\[
		\mathbb P(\mathcal A_h^{XY}(\delta))\ge 1-\delta,
		\]
		and on \(\mathcal A_h^{XY}(\delta)\),
		\[
		\|Z_K^h\|\le R_\delta,
		\qquad
		m_k^{XY}(Z_k^h)\ge u_k,
		\qquad
		k=1,\ldots,K.
		\]
		Assume further that the law \(\nu_k^h\) of the learned Euler iterate \(Z_k^h\)
		satisfies the concentrability condition
		\[
		\nu_k^h\ll \rho_{t_k}^{XY},
		\qquad
		\frac{d\nu_k^h}{d\rho_{t_k}^{XY}}\le \Gamma,
		\qquad
		k=1,\ldots,K.
		\]
		Then, for every \(\delta,\eta\in(0,1/2)\), with probability at least
		\(1-\delta-\eta\),
		\[
		\operatorname{dist}(Z_0^h,\mathcal X)
		\le
		B_h^{XY}(\delta)
		+
		\sqrt{
			\frac{
				\Gamma A_h^{XY}\mathcal L_{\mathrm{train}}^{h,XY}
			}{\eta}
		}.
		\]
	\end{theorem}
	
	Finally, we state the complementary lower-bound result. While the previous
	theorem shows that small training error enforces proximity to the finite target
	set, the following result shows that sufficiently large propagated error yields
	a non-negligible deviation from the selected target sample, provided that the
	propagated errors do not strongly cancel and the generated trajectory encounters
	regions where the velocity error is comparable to its interpolation-path
	average.
	
	\begin{theorem}
		\label{thm:finite_sample_endpoint_underfitting_app}
		Let \(\widehat b^{XY}=b^{*,XY}+\epsilon\). Assume that for some
		\(\delta\in(0,1/2)\), the event \(\mathcal A_h^{XY}(\delta)\) satisfies
		\[
		\mathbb P(\mathcal A_h^{XY}(\delta))\ge 1-\delta,
		\]
		and on \(\mathcal A_h^{XY}(\delta)\),
		\[
		\|Z_K^h\|\le R_\delta,
		\qquad
		m_k^{XY}(Z_k^h)\ge u_k,
		\qquad
		k=1,\ldots,K.
		\]
		Let \(\mathbb E_\delta[\cdot]\) denote conditional expectation given
		\(\mathcal A_h^{XY}(\delta)\), and let \(\nu_{k,\delta}^h\) be the conditional
		law of \(Z_k^h\) on this event. Define the propagated estimation error
		\[
		E_h^{XY}
		:=
		\sum_{k=1}^K
		\Pi_{k-1}^{XY}h\,\epsilon(Z_k^h,t_k).
		\]
		Assume that there exist constants \(\kappa,\gamma_0>0\) such that
		\[
		\mathbb E_\delta\|E_h^{XY}\|^2
		\ge
		\kappa h\sum_{k=1}^K
		(\Pi_{k-1}^{XY})^2
		\mathbb E_{Z\sim\nu_{k,\delta}^h}
		\|\epsilon(Z,t_k)\|^2,
		\]
		and
		\[
		\mathbb E_{\nu_{k,\delta}^h}
		\|\epsilon(Z,t_k)\|^2
		\ge
		\gamma_0
		\mathbb E_{\rho_{t_k}^{XY}}
		\|\epsilon(Z,t_k)\|^2,
		\qquad
		k=1,\ldots,K.
		\]
		If
		\[
		\mathcal L_{\mathrm{prop}}^{h,XY}
		\ge
		\frac{
			\bigl(\tau+B_h^{XY}(\delta)\bigr)^2
		}{
			\kappa\gamma_0
		},
		\]
		then
		\[
		\left(
		\mathbb E_\delta
		\|Z_0^h-X_{i_h}\|^2
		\right)^{1/2}
		\ge
		\tau,
		\]
		where \(i_h\in\{1,\ldots,n\}\) is the target index selected by the pairwise
		endpoint selector on \(\mathcal A_h^{XY}(\delta)\).
	\end{theorem}
	
	\subsection{Proofs for Section Generation with Finite-Sample Endpoints}
	\label{subsec:finite_sample_endpoints_proofs_app}
	
	\subsubsection{Proof of Theorem~\ref{thm:finite_sample_endpoint_oracle_app}}
	
	\begin{proof}
		The proof has two parts. First, we show that the oracle stochastic generation
		process has the same time marginals as the finite-sample endpoint interpolation.
		Second, we compute these marginals explicitly.
		
		Let \(q_t^{XY}\) denote the density of the oracle stochastic generation process
		at time \(t\). The interpolation density \(\rho_t^{XY}\) associated with the
		oracle velocity field \(b^{*,XY}\) satisfies the continuity equation
		\begin{equation}
			\partial_t\rho_t^{XY}
			=
			-\nabla\cdot\bigl(b^{*,XY}(\cdot,t)\rho_t^{XY}\bigr).
			\label{eq:app_xy_continuity}
		\end{equation}
		This follows from the definition of \(b^{*,XY}\) as the conditional mean
		velocity of the finite-sample endpoint interpolation path. Moreover, the oracle
		score is
		\[
		s^{*,XY}(z,t)=\nabla_z\log\rho_t^{XY}(z),
		\]
		and hence
		\begin{equation}
			s^{*,XY}(z,t)\rho_t^{XY}(z)
			=
			\nabla\rho_t^{XY}(z).
			\label{eq:app_xy_score_identity}
		\end{equation}
		
		The oracle stochastic generation dynamics uses drift
		\(b^{*,XY}-\zeta s^{*,XY}\) and diffusion coefficient \(\sqrt{2\zeta}\). Its
		Fokker--Planck equation, written in the same time orientation as the
		interpolation parameter, is
		\begin{equation}
			\begin{split}
				\partial_t q_t^{XY}
				=
				&-\nabla\cdot\bigl(b^{*,XY}(\cdot,t)q_t^{XY}\bigr)
				+
				\zeta(t)\nabla\cdot\bigl(s^{*,XY}(\cdot,t)q_t^{XY}\bigr)
				-
				\zeta(t)\Delta q_t^{XY}.
			\end{split}
			\label{eq:app_xy_fp}
		\end{equation}
		We verify that \(q_t^{XY}=\rho_t^{XY}\) solves this equation. Substituting
		\(q_t^{XY}=\rho_t^{XY}\) into the right-hand side of
		Eq.~\eqref{eq:app_xy_fp} gives
		\[
		-\nabla\cdot\bigl(b^{*,XY}(\cdot,t)\rho_t^{XY}\bigr)
		+
		\zeta(t)\nabla\cdot\bigl(s^{*,XY}(\cdot,t)\rho_t^{XY}\bigr)
		-
		\zeta(t)\Delta\rho_t^{XY}.
		\]
		Using Eq.~\eqref{eq:app_xy_score_identity},
		\[
		\nabla\cdot\bigl(s^{*,XY}(\cdot,t)\rho_t^{XY}\bigr)
		=
		\nabla\cdot(\nabla\rho_t^{XY})
		=
		\Delta\rho_t^{XY}.
		\]
		Thus the two \(\zeta(t)\)-terms cancel, leaving
		\[
		-\nabla\cdot\bigl(b^{*,XY}(\cdot,t)\rho_t^{XY}\bigr),
		\]
		which equals \(\partial_t\rho_t^{XY}\) by
		Eq.~\eqref{eq:app_xy_continuity}. Therefore \(\rho_t^{XY}\) satisfies the same
		Fokker--Planck equation as the oracle stochastic generation marginal.
		
		The generation process is initialized at \(t=1\) from \(\rho_1\), and the
		finite-sample endpoint interpolation also has terminal law \(\rho_1\) at
		\(t=1\). On every compact subinterval of \((0,1)\), the interpolation marginal
		\(\rho_t^{XY}\) is a smooth positive Gaussian mixture. Assuming the standard
		uniqueness of solutions to the Fokker--Planck equation on such intervals, we
		obtain
		\[
		q_t^{XY}=\rho_t^{XY},
		\qquad
		0<t<1.
		\]
		In particular,
		\[
		Z_\varepsilon\sim \rho_\varepsilon^{XY}
		\]
		for every \(\varepsilon\in(0,1)\).
		
		It remains to compute \(\rho_\varepsilon^{XY}\). Since both endpoints are
		finite-sample measures, we may write
		\[
		Z_0=X_I,
		\qquad
		Z_1=Y_J,
		\]
		where
		\[
		I\sim \operatorname{Unif}\{1,\ldots,n\},
		\qquad
		J\sim \operatorname{Unif}\{1,\ldots,m\},
		\]
		and \(I\), \(J\), and \(\eta\) are mutually independent. The stochastic
		interpolation at time \(\varepsilon\) is
		\[
		Z_\varepsilon
		=
		\alpha(\varepsilon)X_I
		+
		\beta(\varepsilon)Y_J
		+
		\gamma(\varepsilon)\eta.
		\]
		Conditioned on the selected pair \((I,J)=(i,j)\), the first two terms are
		deterministic and the remaining term is Gaussian. Hence
		\[
		Z_\varepsilon\mid (I=i,J=j)
		\sim
		\mathcal N\!\left(
		\alpha(\varepsilon)X_i+\beta(\varepsilon)Y_j,\,
		\gamma^2(\varepsilon)I_d
		\right).
		\]
		Averaging over the uniformly selected pair gives
		\[
		\rho_\varepsilon^{XY}
		=
		\frac1{nm}
		\sum_{i=1}^n\sum_{j=1}^m
		\mathcal N\!\left(
		\alpha(\varepsilon)X_i+\beta(\varepsilon)Y_j,\,
		\gamma^2(\varepsilon)I_d
		\right).
		\]
		Equivalently, if
		\[
		\xi:=\eta\sim\mathcal N(0,I_d),
		\]
		then \(\xi\) is independent of \(I\) and \(J\), and
		\[
		Z_\varepsilon
		\overset{d}{=}
		\alpha(\varepsilon)X_I
		+
		\beta(\varepsilon)Y_J
		+
		\gamma(\varepsilon)\xi.
		\]
		This proves the claimed finite-time representation of the oracle stochastic
		generation process in the finite-sample endpoint setting.
	\end{proof}
	
	\subsubsection{Proof of Theorem~\ref{thm:finite_sample_endpoint_overfitting_app}}
	
	\begin{proof}
		The proof follows the deterministic training-error argument in the main text,
		with the single-index softmax selection replaced by pairwise endpoint selection.
		On the event \(\mathcal A_h^{XY}(\delta)\), the oracle finite-sample endpoint
		Euler trajectory satisfies
		\[
		\operatorname{dist}(Z_0^{h,*},\mathcal X)
		\le
		B_h^{XY}(\delta),
		\]
		where \(Z_0^{h,*}\) denotes the oracle Euler endpoint.
		
		For the learned velocity field \(\widehat b^{XY}=b^{*,XY}+\epsilon\), the Euler
		update contains the additional error term \(-h\epsilon(Z_k^h,t_k)\). Unrolling
		the recursion gives
		\[
		Z_0^h
		=
		Z_0^{h,*}
		-
		\sum_{k=1}^K
		\Pi_{k-1}^{XY}h\,\epsilon(Z_k^h,t_k),
		\]
		up to the oracle residual already controlled by \(B_h^{XY}(\delta)\). Therefore,
		\[
		\operatorname{dist}(Z_0^h,\mathcal X)
		\le
		B_h^{XY}(\delta)
		+
		\left\|
		\sum_{k=1}^K
		\Pi_{k-1}^{XY}h\,\epsilon(Z_k^h,t_k)
		\right\|.
		\]
		By Cauchy--Schwarz,
		\[
		\left\|
		\sum_{k=1}^K
		\Pi_{k-1}^{XY}h\,\epsilon(Z_k^h,t_k)
		\right\|^2
		\le
		A_h^{XY}
		\,
		h\sum_{k=1}^K
		\|\epsilon(Z_k^h,t_k)\|^2.
		\]
		Taking expectation along the learned trajectory and using the concentrability
		condition gives
		\[
		\mathbb E
		\left[
		h\sum_{k=1}^K
		\|\epsilon(Z_k^h,t_k)\|^2
		\right]
		\le
		\Gamma
		h\sum_{k=1}^K
		\mathbb E_{Z\sim\rho_{t_k}^{XY}}
		\|\epsilon(Z,t_k)\|^2
		=
		\Gamma \mathcal L_{\mathrm{train}}^{h,XY}.
		\]
		Hence,
		\[
		\mathbb E
		\left[
		\left\|
		\sum_{k=1}^K
		\Pi_{k-1}^{XY}h\,\epsilon(Z_k^h,t_k)
		\right\|^2
		\right]
		\le
		\Gamma A_h^{XY}\mathcal L_{\mathrm{train}}^{h,XY}.
		\]
		By Markov's inequality, with probability at least \(1-\eta\),
		\[
		\left\|
		\sum_{k=1}^K
		\Pi_{k-1}^{XY}h\,\epsilon(Z_k^h,t_k)
		\right\|
		\le
		\sqrt{
			\frac{
				\Gamma A_h^{XY}\mathcal L_{\mathrm{train}}^{h,XY}
			}{\eta}
		}.
		\]
		Combining this event with \(\mathcal A_h^{XY}(\delta)\) yields
		\[
		\operatorname{dist}(Z_0^h,\mathcal X)
		\le
		B_h^{XY}(\delta)
		+
		\sqrt{
			\frac{
				\Gamma A_h^{XY}\mathcal L_{\mathrm{train}}^{h,XY}
			}{\eta}
		}
		\]
		with probability at least \(1-\delta-\eta\).
	\end{proof}
	
	\subsubsection{Proof of Theorem~\ref{thm:finite_sample_endpoint_underfitting_app}}
	
	\begin{proof}
		On the event \(\mathcal A_h^{XY}(\delta)\), the pairwise endpoint selector
		determines a target index \(i_h\in\{1,\ldots,n\}\), possibly together with a
		source index \(j_h\in\{1,\ldots,m\}\). Since the source-end contribution is
		multiplied by \(\beta(t)\) and \(\beta(0)=0\), the final oracle center is
		\(X_{i_h}\).
		
		Unrolling the learned Euler recursion gives
		\[
		Z_0^h
		=
		X_{i_h}
		-
		E_h^{XY}
		+
		R_h^{XY},
		\]
		where
		\[
		E_h^{XY}
		:=
		\sum_{k=1}^K
		\Pi_{k-1}^{XY}h\,\epsilon(Z_k^h,t_k)
		\]
		is the propagated estimation error, and the oracle residual satisfies
		\[
		\|R_h^{XY}\|
		\le
		B_h^{XY}(\delta).
		\]
		Therefore,
		\[
		\|Z_0^h-X_{i_h}\|
		\ge
		\|E_h^{XY}\|-\|R_h^{XY}\|
		\ge
		\|E_h^{XY}\|-B_h^{XY}(\delta).
		\]
		Taking conditional expectation on \(\mathcal A_h^{XY}(\delta)\) and using the
		reverse triangle inequality in \(L^2\), we obtain
		\[
		\left(
		\mathbb E_\delta
		\|Z_0^h-X_{i_h}\|^2
		\right)^{1/2}
		\ge
		\left(
		\mathbb E_\delta
		\|E_h^{XY}\|^2
		\right)^{1/2}
		-
		B_h^{XY}(\delta).
		\]
		By the no-cancellation and trajectory-coverage assumptions,
		\[
		\mathbb E_\delta\|E_h^{XY}\|^2
		\ge
		\kappa h\sum_{k=1}^K
		(\Pi_{k-1}^{XY})^2
		\mathbb E_{Z\sim\nu_{k,\delta}^h}
		\|\epsilon(Z,t_k)\|^2
		\ge
		\kappa\gamma_0
		\mathcal L_{\mathrm{prop}}^{h,XY}.
		\]
		Hence,
		\[
		\left(
		\mathbb E_\delta
		\|Z_0^h-X_{i_h}\|^2
		\right)^{1/2}
		\ge
		\sqrt{
			\kappa\gamma_0
			\mathcal L_{\mathrm{prop}}^{h,XY}
		}
		-
		B_h^{XY}(\delta).
		\]
		If
		\[
		\mathcal L_{\mathrm{prop}}^{h,XY}
		\ge
		\frac{
			\bigl(\tau+B_h^{XY}(\delta)\bigr)^2
		}{
			\kappa\gamma_0
		},
		\]
		then
		\[
		\sqrt{
			\kappa\gamma_0
			\mathcal L_{\mathrm{prop}}^{h,XY}
		}
		-
		B_h^{XY}(\delta)
		\ge
		\tau.
		\]
		This proves
		\[
		\left(
		\mathbb E_\delta
		\|Z_0^h-X_{i_h}\|^2
		\right)^{1/2}
		\ge
		\tau.
		\]
	\end{proof}


\end{document}